\documentclass[letterpaper]{article} 
\usepackage{aaai23}  
\usepackage{times}  
\usepackage{helvet}  
\usepackage{courier}  
\usepackage[hyphens]{url}  
\usepackage{graphicx} 
\usepackage{natbib}  
\usepackage{caption} 
\usepackage{algorithm}
\usepackage{algorithmic}

\usepackage{newfloat}
\usepackage{listings}
\DeclareCaptionStyle{ruled}{labelfont=normalfont,labelsep=colon,strut=off} 
\floatstyle{ruled}
\newfloat{listing}{tb}{lst}{}
\floatname{listing}{Listing}
\title{Efficient Exploration in Resource-Restricted Reinforcement Learning}
\author{
Zhihai Wang\textsuperscript{\rm 1}\thanks{Equal contribution.}, 
Taoxing Pan\textsuperscript{\rm 1}\footnotemark[1], 
Qi Zhou\textsuperscript{\rm 1},
Jie Wang\textsuperscript{\rm 1,2}\thanks{Corresponding author.}
}
\affiliations{
\textsuperscript{\rm 1}CAS Key Laboratory of Technology in GIPAS, 
University of Science and Technology of China\\
\textsuperscript{\rm 2}Institute of Artificial Intelligence, Hefei Comprehensive National Science Center\\
\{zhwangx, tx1997,zhouqida\}@mail.ustc.edu.cn\\
jiewangx@ustc.edu.cn
}

\usepackage{bibentry}
\usepackage{amssymb}
\usepackage{caption}
\usepackage{subcaption}
\usepackage{amsmath}
\usepackage{amsthm}
\usepackage{booktabs}
\usepackage{algorithm}
\usepackage{algorithmic}

\newtheorem{lemma}{Lemma}

\newtheorem{theorem}{Theorem}

\DeclareMathOperator*{\argmax}{arg\,max}

\begin{document}

\maketitle

\begin{abstract}

In many real-world applications of reinforcement learning (RL), performing actions requires consuming certain types of resources that are non-replenishable in each episode. Typical applications include robotic control with limited energy and video games with consumable items. In tasks with non-replenishable resources, we observe that popular RL methods such as soft actor critic suffer from poor sample efficiency. The major reason is that, they tend to exhaust resources fast and thus the subsequent exploration is severely restricted due to the absence of resources. To address this challenge, we first formalize the aforementioned problem as a resource-restricted reinforcement learning, and then propose a novel \textbf{r}esource-\textbf{a}ware \textbf{e}xploration \textbf{b}onus (RAEB) to make reasonable usage of resources. An appealing feature of RAEB is that, it can significantly reduce unnecessary resource-consuming trials while effectively encouraging the agent to explore unvisited states. Experiments demonstrate that the proposed RAEB significantly outperforms state-of-the-art exploration strategies in resource-restricted reinforcement learning environments, improving the sample efficiency by up to an order of magnitude.
\end{abstract}

\section{Introduction}



Performing actions requires consuming certain types of resources in many real-world decision-making tasks \cite{electric_robots,kempka2016vizdoom,bhatia2019resource,resource_management_cloud}. For example, the availability of the action ``jump'' depends on the remaining energy in robotic control \cite{grizzle2014models}. Another example is that the availability of a specific skill depends on a certain type of in-game items in video games \cite{vinyals2019grandmaster}.
Therefore, making reasonable usage of resources is one of the keys to success in decision-making tasks with limited resources. \cite{bhatia2019resource}.
In recent years, reinforcement learning (RL) has achieved success in decision-making tasks from games to robotic control in simulators \cite{trpo, silver2017mastering}. However, RL with limited resources has not been well studied, and RL methods can struggle to make reasonable usage of resources. 
In this paper, we take the first step towards studying RL with resources that are non-replenishable in each episode, which can be extremely challenging and is very common in the real world. Typical applications include robotic control with non-replenishable energy \cite{electric_robots,kormushev2013reinforcement}
and  video games with non-replenishable in-game items \cite{kempka2016vizdoom, resnick2018pommerman}.

This paper begins by evaluating several popular RL algorithms, including proximal policy optimization (PPO) \cite{PPO} and soft actor critic (SAC) \cite{pmlr-v80-haarnoja18b}, in various tasks with non-replenishable resources. We find these algorithms struggle to explore the environments efficiently and suffer from poor sample efficiency. Moreover, we empirically show that the surprise-based exploration method \cite{Achiam17}, one of the state-of-the-art advanced exploration strategies, still suffers from poor sample efficiency. Even worse, we observe that some of these algorithms struggle to perform better than random agents in these tasks. We further perform an in-depth analysis of this challenge, and find that the exploration is severely restricted by the resources. As the available actions depend on the remaining resources and these algorithms tend to exhaust resources rapidly, the subsequent resource-consuming exploration is severely restricted. However, resource-consuming actions are usually significant for achieving high rewards, such as consuming consumable items in video games. Therefore, these algorithms suffer from inefficient exploration. (See Section \ref{sec:inefficient_exploration})





To address this challenge, we first formalize the aforementioned problems as a \textbf{r}esource-\textbf{r}estricted \textbf{r}einforcement \textbf{l}earning (R3L), where the available actions largely depend on the remaining resources. Specifically, we augment the Markov Decision Process \cite{reinforcement_learning} with resource-related information that is easily accessible in many real-world tasks, including a key map from a state to its remaining resources. 
We then propose a novel \textbf{r}esource-\textbf{a}ware \textbf{e}xploration \textbf{b}onus (RAEB) that significantly improves the exploration efficiency by making reasonable usage of resources. Based on the observation that the accessible state set of a given state---which the agent can possibly reach from the given state---largely depends on the remaining resources of the given state and large accessible state sets are essential for efficient exploration in these tasks, RAEB encourages the agent to explore unvisited states that have large accessible state sets. Specifically, we quantify the RAEB of a given state by its novelty and remaining resources, which simultaneously promotes novelty-seeking and resource-saving exploration. 

To compare RAEB with the baselines, we design a range of robotic delivery and autonomous electric robot tasks based on  Gym \cite{gym} and Mujoco \cite{mujoco}. In these tasks, we regard goods and/or electricity as resources (see Section \ref{sec:inefficient_exploration}). Experiments demonstrate that the proposed RAEB significantly outperforms state-of-the-art exploration strategies in several challenging R3L environments, improving the sample efficiency by up to an order of magnitude. Moreover, we empirically show that our proposed approach significantly reduces unnecessary resource-consuming trials while effectively encouraging the agent to explore unvisited states.

\section{Related Work}
\textbf{Decision-making tasks with limited resources} 
Some work has studied the applications of reinforcement learning in specific resource-related tasks, such as job scheduling \cite{zhang1995reinforcement,resource_management_cloud}, resource allocation \cite{tesauro2005online}, and unpowered glider soaring \cite{learning_to_soar}. However, they are application-customized, and thus general problems of RL with limited resources have not been well studied. In this paper, we take the first step towards studying RL with resources that are non-replenishable in each episode.

\noindent\textbf{Intrinsic reward-based exploration}
Efficient exploration remains a major challenge in reinforcement learning. It is common to generate intrinsic rewards to guide efficient exploration. Existing intrinsic reward-based exploration methods include count-based \cite{MBIE,NIPS2016_6383}, prediction-error-based \cite{pmlr-v70-pathak17a,Achiam17}, information-gain-based \cite{VIME, max}, and empowerment-based \cite{empowerment} methods. Our proposed RAEB is also an intrinsic reward. However, RAEB promotes efficient exploration in R3L tasks while previous exploration methods struggle to explore R3L environments efficiently.

\noindent\textbf{Constrained reinforcement learning}
Constrained reinforcement learning is an active topic in RL research. Constrained reinforcement learning methods \cite{cpo,Primal-Dual,pmlr-v129-chen20b} address the challenge of learning policies that maximize the expected return while satisfying the constraints. Although there exist resource constraints that the quantity of resources is limited in R3L problems, the constraints can be easily introduced to the environment or the agent. For example, we can set the resource-consuming actions unavailable if the resources are exhausted in R3L environments. Thus, any policy can satisfy the resource constraints. In contrast, the major challenge of R3L problems lies in efficient exploration instead of finding policies that satisfy  resource constraints.


\section{Preliminaries}
We introduce the notation we will use throughout the paper. We consider an infinite horizon Markov Decision Process (MDP) denoted by a tuple $(\mathcal{S}, \mathcal{A}, p, r, \gamma)$, where the state space $\mathcal{S}\subset \mathbb{R}^m$ and the action space $\mathcal{A} \subset \mathbb{R}^n$ are continuous, $p: \mathcal{S}\times \mathcal{A} \times \mathcal{S}\to [0, \infty)$ is the transition probability distribution, $r: \mathcal{S}\times \mathcal{A} \to [r_{\min} , r_{\max} ]$ is the reward function, $\gamma \in (0,1)$ is a discount factor. Let the policy $\pi$ maps each state $s\in \mathcal{S}$ to a probability distribution over the action space $\mathcal{A}$. That is, $\pi(\cdot|s)$ is the probability density function (PDF) over the action space $\mathcal{A}$. Define the set of feasible policies as $\Pi$. 
In reinforcement learning, we aim to find the policy $\pi$ that maximizes the cumulative rewards  

\vspace{-0.3cm}
\begin{align}
\label{eqn:objective}
\pi^* = \argmax_{\pi\in \Pi} \mathbb{E}_\pi \left[\sum_{t=0}^\infty \gamma^t r(s_t,a_t)\right],
\end{align}
where $s_0 \sim \mu$, $a_t \sim \pi(\cdot| s_t)$, $s_{t+1} \sim p(\cdot| s_t, a_t)$, and $\mu$ is the initial distribution of state.

We define the probability density of the next state $s^\prime$ after one step transition following the policy $\pi$ as

\vspace{-0.3cm}
\begin{align*}
p^\pi_1(s^\prime| s) = \int_{a\in \mathcal{A}} p(s^\prime| s,a) \pi(a| s)da.
\end{align*}
Then, we recursively calculate the probability density of the state $s^\prime$ after $t + 1$ steps transition following the policy $\pi$ by

\vspace{-0.3cm}
\begin{align*}
p^\pi_{t+1}(s^\prime| s) = \iint_{s_t,a_t} p(s^\prime| s_{t},a_{t}) &\pi(a_{t}| s_{t})\\
&p_{t}^\pi(s_{t}| s)da_{t} ds_t.
\end{align*}

\section{Resource-Restricted \\ Reinforcement Learning (R3L)}




\begin{figure*}[t]
    \centering
    \begin{subfigure}{0.28\textwidth}
        \includegraphics[width=\textwidth]{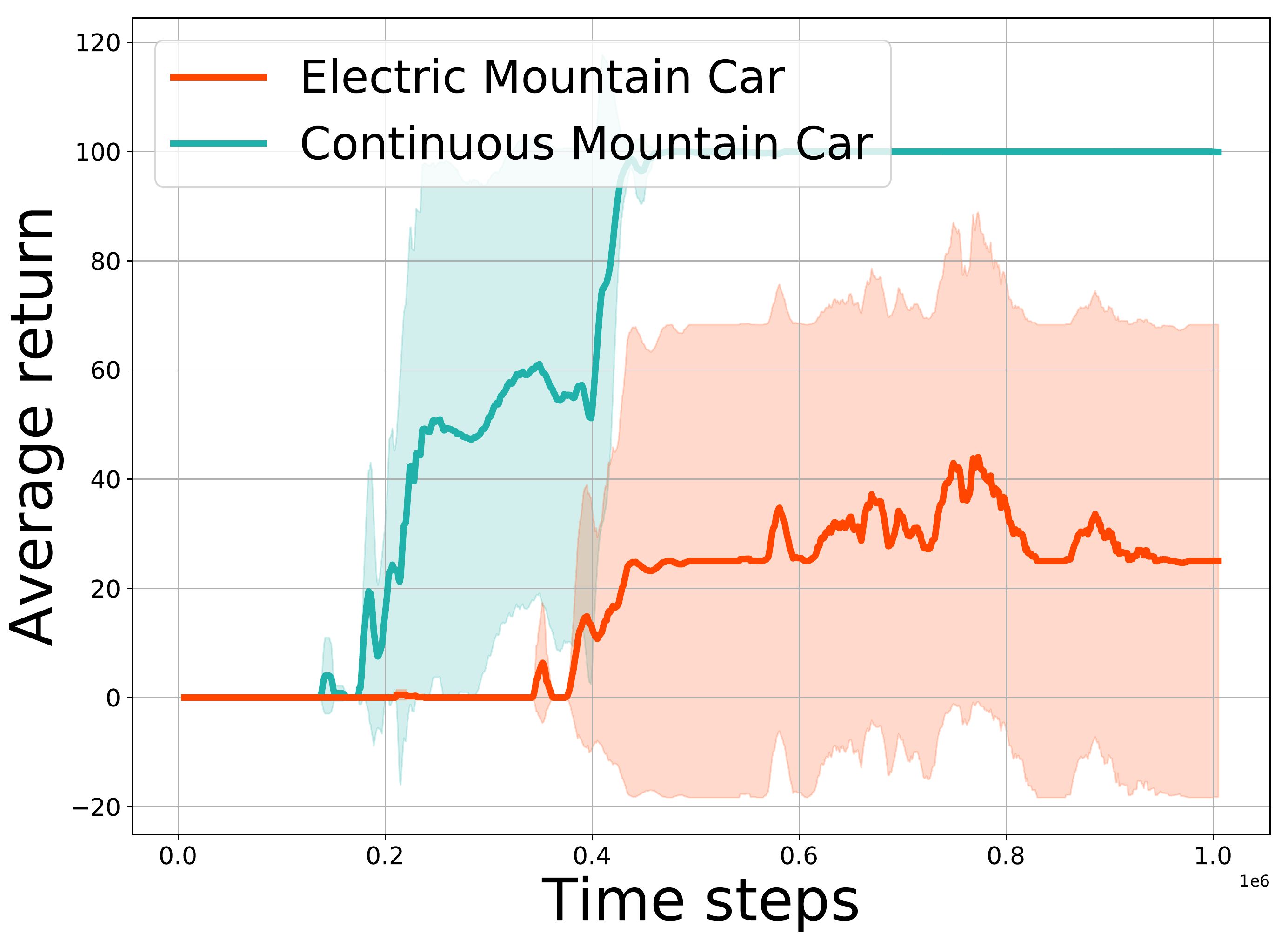}
        \caption{Performance of Surprise method on two Mountain Car environments.}
        \label{fig:car_with_or_without_fuel}
    \end{subfigure}
    \hspace{5mm}
    \begin{subfigure}{0.28\textwidth}
        \includegraphics[width=\textwidth]{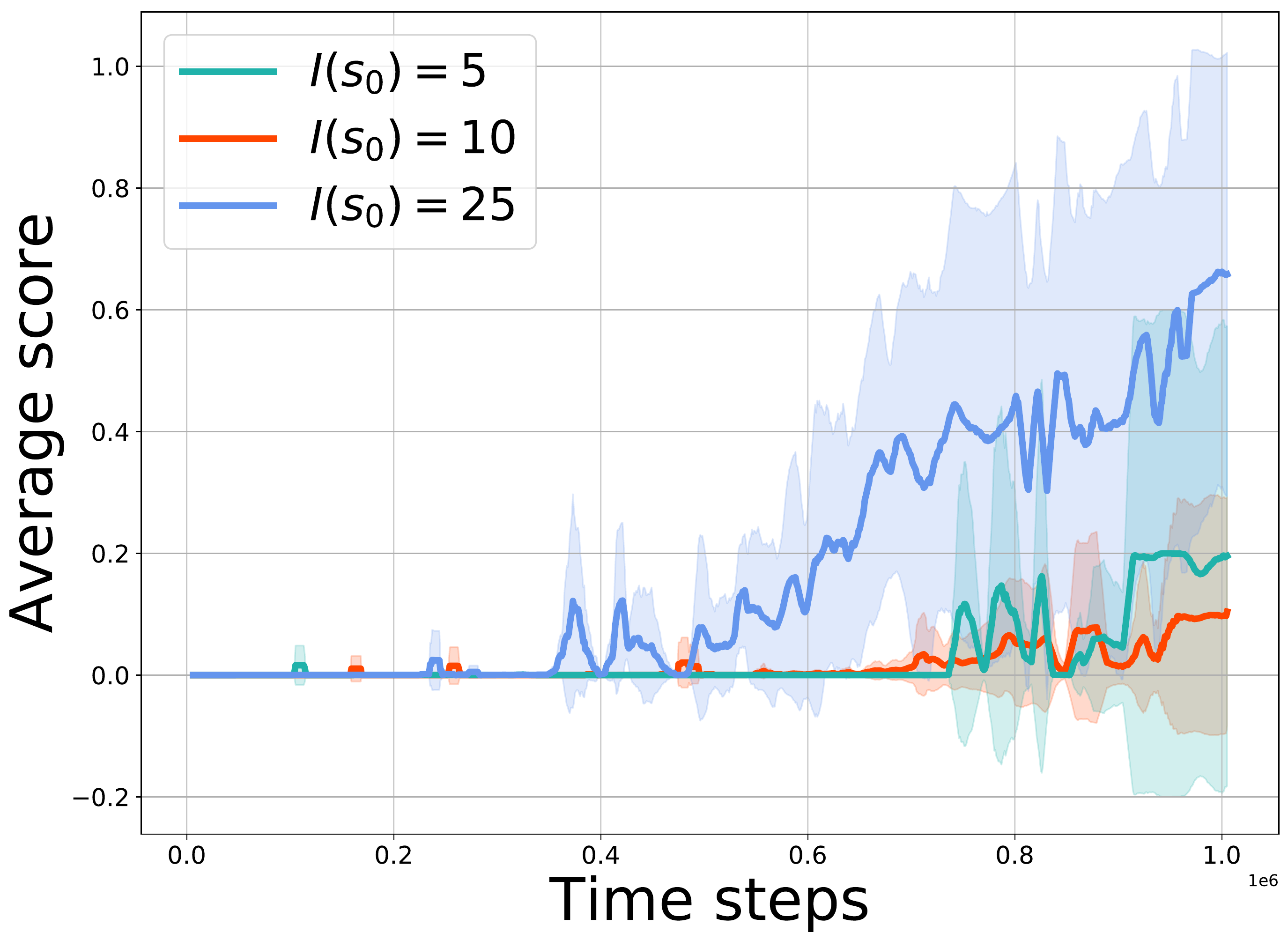}
        \caption{Performance of Surprise method with different initial resources.}
        \label{fig:Surprise_car_decreasing_resource}   
    \end{subfigure}
    \hspace{5mm}
    \begin{subfigure}{0.28\textwidth}
        \includegraphics[width=\textwidth]{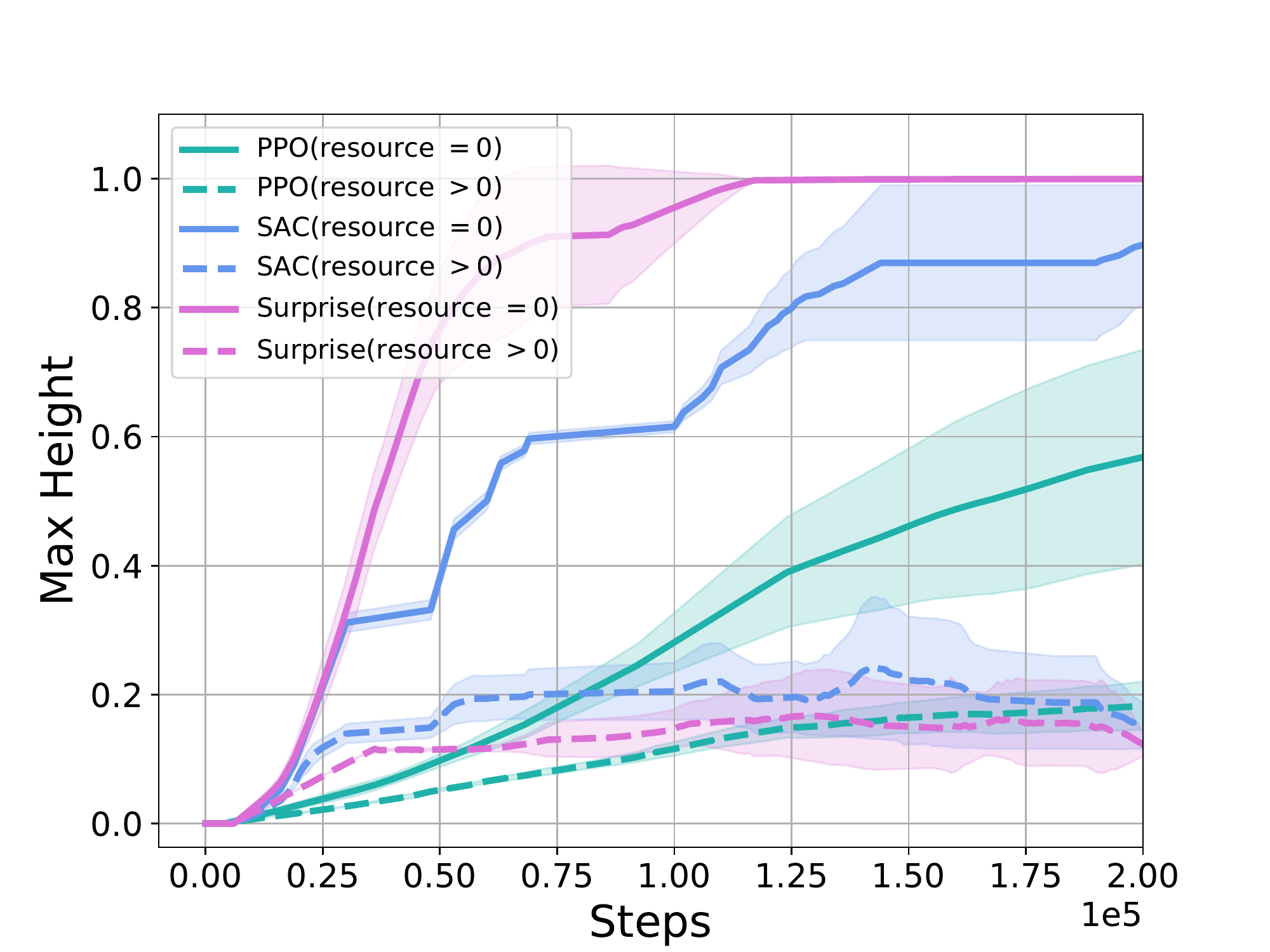}
        \caption{Learning behavior of PPO, SAC, and Surprise on Delivery Mountain Car.}
        \label{fig:Exploration behavior on Delivery Mountain Car}   
    \end{subfigure}
    \vspace{-2mm}
    \caption{The solid curves correspond to the mean and the shaded region to the standard deviation over at least four random seeds. For visual clarity, we smooth curves uniformly. (a) The results show that the average return of Surprise on Electric Mountain Car is much lower than that on Continuous Mountain Car. (b) The results show that the average score of Surprise on Delivery Mountain Car increases with the initial number of goods.
    (c) We evaluate the learned policies every 10000 training steps on Delivery Mountain Car. 
    The dashed curves correspond to the max height that the car can reach before exhausting the resources, while the solid curves correspond to the max height that the car can reach regardless of whether the car exhausts the resources or not. The results demonstrate that PPO, SAC, and Surprise exhaust the resources when the car only reaches a low height, and then the subsequent exploration is severely restricted due to the absence of resources.}
    \label{fig:reason_for_inefficient_exploration}
    \vspace{-3mm}
\end{figure*}



We present a detailed formulation of RL with limited resources in this section. We empirically show that the resource-related information is critical for efficient exploration, as the available actions largely depend on the remaining resources (see Section \ref{sec:inefficient_exploration}).  However, the general reinforcement learning formulation assumes the environment is totally unknown and neglects the resource-related information. To tackle this problem, we formalize the tasks with resources as a resource-restricted reinforcement learning, which introduces accessible resource-related information. 

\subsection{A detailed formulation of R3L}

We define certain objects that performing actions requires as resources in reinforcement learning settings, such as the energy in robotic control and consumable items in video games. Suppose that there are $d$ types of resources. We use the notation $\xi \in \mathbb{R}^d_{+}$ to denote the resource vector and $\mathcal{S}_r \subset \mathbb{R}^d_{+}$ to denote the set of all possible resource vectors. 

We augment the state space $\mathcal{S}$ with $\mathcal{S}_r$. That is, the state space in R3L is $\mathcal{S}_{\text{r3l}} = \{s|s=\left[s_o,s_r\right], \forall s_o\in \mathcal{S},\forall s_r \in \mathcal{S}_r\}$, where $\left[s_o, s_r\right]$ denotes the concatenation of $s_o$ and $s_r$. With this augmented state space, algorithms can learn resource-related information from data. 

To better exploit the resource-related information, we define a resource-aware function $I:\mathcal{S}\to \mathbb{R}^d_{+}$ from the state to the quantity of its available resources,

\vspace{-0.4cm}
\begin{align*}
 I(s) = (I_1(s), \dots, I_d(s)),\, I_i(s) \geq 0,\,\forall i \in [d],
\end{align*}
where $[d] = \{1,2,\dots, d\}$. We call the $i$-th type of resources non-replenishable if the quantity is monotonically nonincreasing in an episode, i.e., $I_i(s_t) \geq I_i(s_{t+1})$. We assume $I$ is known as a priori because the information about the current resources is easily accessible in many real-world tasks. 

For completeness, we further define a deterministic resource transition function $f: \mathcal{S}_{\text{r3l}} \times \mathcal{A} \to \mathcal{S}_r$ which represents the transition function of the resources. We assume $f$ is unknown as the dynamic of the resources is often inaccessible in real-world tasks. Thus, the transition probability distribution is defined by $p_{r3l}(s^{\prime}|s,a) = p(s_o^{\prime}|s_o,a) \mathbb{I}(s_r^{\prime} = f(s,a))$, where $s=\left[s_o,s_r\right]$, $s^{\prime} = \left[s_o^{\prime}, s_r^{\prime}\right]$, and $\mathbb{I}(\cdot)$ denotes the indicator function. For simplicity, we also denote by $p$ the transition probability distribution in R3L problems. 

We call the RL problems with limited resources as R3L problems, whose formulation is denoted by a tuple $(\mathcal{S}_{\text{r3l}}, \mathcal{A}, p, r, \gamma, I, f)$. Let $\mathcal{A}(s)$ denote the available action set of a given state $s$, and we have $\cup_{s\in \mathcal{S}_{\text{r3l}}}\mathcal{A}(s)=\mathcal{A}$. Note that the $\mathcal{A}(s)$ depends on the remaining resources at state $s$. The feasible policy $\pi(\cdot|s)$ is a PDF over the action space $\mathcal{A}(s)$. We also denote by $\Pi$ the feasible policy set in R3L problems. The same as conventional RL, R3L also aims to solve the problem (\ref{eqn:objective}) to find the optimal feasible policy.

Moreover, we call the state $s_2$ is accessible from the state $s_1$, if there is a stationary policy $\pi \in \Pi$ and a time step $t\in\mathbb{N}_+$, such that the probability density $p^\pi_t(s_2| s_1) >0$. Suppose the resources are non-replenishable and $\exists i \in [d], I_i(s_1) < I_i(s_2)$, then $s_2$ is inaccessible from $s_1$.  Besides, we define the set of accessible states of a state $s$ as 

\vspace{-0.3cm}
$$A_c(s) \dot= \{s^\prime:s^\prime  \text{ is accessible from } s\}.$$

\section{Challenge in R3L Tasks}\label{sec:inefficient_exploration}

In general, we divide real-world decision-making tasks with non-replenishable resources into two categories. In the first category of tasks, all actions consume resources and different actions consume different quantities of resources, such as robotic control with limited energy. In these tasks, the agent needs to seek actions that achieve high rewards while consuming small quantities of resources. In the second category, only specific actions consume resources, such as video games with consumable items. In these tasks, the agent needs to seek proper states to consume the resources. 

To evaluate popular RL methods in both kinds of R3L tasks, we design three series of environments with limited resources based on Gym \cite{gym} and Mujoco \cite{mujoco}. The first is the autonomous electric robot task. In this task, the resource is electricity and all actions consume electricity. The quantity of consumed electricity depends on the amplitude of the action $a$, defined by $0.1*\|a\|_2^2$. 
The second is the robotic delivery task. In this task, the resource is goods and only the ``unload'' action consumes the goods. The agent needs to ``unload'' the goods at an unknown destination. The third is a task that combines the first and the second. Please refer to Appendix B.1 for more details about these environments. Specifically, in this part, we use two designed R3L tasks, namely Electric Mountain Car and Delivery Mountain Car. The two tasks are based on the classic control environment---Continuous Mountain Car. On Electric Mountain Car, the agent aims to reach the mountain top with limited electricity. On Delivery Mountain Car, the agent aims to deliver the goods to the mountain top. 

\subsection{Inefficient exploration in R3L tasks}\label{section4:continuous environments}


%
Based on the designed R3L tasks, we first show that a state-of-the-art exploration method, i.e., the surprise-based exploration method (Surprise) \cite{Achiam17}, suffers from poor sample efficiency in both kinds of R3L tasks, especially those with scarce resources. We compare the performance of Surprise on Electric Mountain Car and Continuous Mountain Car (without electricity restriction) as shown in Figure \ref{fig:car_with_or_without_fuel}. The results show that the performance of Surprise on Electric Mountain Car is much poorer than that on Continuous Mountain Car. We further compare the performance of Surprise on Delivery Mountain Car with different initial quantities of goods. Figure \ref{fig:Surprise_car_decreasing_resource} shows that the performance of Surprise improves with the initial quantity of goods. In particular, Surprise struggles to learn a policy better than random agents when resources are scarce. 




To provide further insight into the challenge in R3L tasks, we then analyze the learning behavior of several popular RL methods, i.e., PPO, SAC, and Surprise, on Delivery Mountain Car.  Figure \ref{fig:Exploration behavior on Delivery Mountain Car} shows their learning behavior for the first $2\times 10^5$ time steps. Figure \ref{fig:Exploration behavior on Delivery Mountain Car} shows that the max height---which the car can reach during an episode regardless of whether the car exhausts the resources or not---is much higher than the max height that the car can reach before exhausting the resources. 
That is, these methods exhaust the resources when the car only reaches a low height, and then the subsequent exploration of the ``unload'' action is severely restricted due to the absence of resources. Therefore, these methods have difficulty in achieving the goal of delivering the goods to the mountain top, though they can reach the mountain top. Overall, these methods suffer from inefficient exploration in R3L problems.

\section{Methods}
In this section, we focus on promoting efficient exploration in R3L problems. We first present a detailed description of our proposed \textbf{r}esource-\textbf{a}ware \textbf{e}xploration \textbf{b}onus (RAEB) and then theoretically analyze the efficiency of RAEB in the finite horizon tabular setting. We summarize the procedure of RAEB in Algorithm \ref{alg:my_algorithm}.


\subsection{Resource-Aware Exploration Bonus}\label{sec:raeb}

As shown in Section \ref{sec:inefficient_exploration}, resource plays an important role in efficient exploration in R3L problems. We observe that the size of accessible state sets of a given state generally positively correlates with its available resources in R3L problems with non-replenishable resources, as states with high resources are inaccessible from states with low resources. Previous empowerment-based exploration methods \cite{empowerment, empowerment_gain} have shown that exploring states with large accessible state sets is essential for efficient exploration. That is, moving towards states with high resources enables the agent to reach a large number of future states, thus exploring the environment efficiently. However, many existing exploration methods such as Surprise neglect the resource-related information and exhaust resources fast. Thus, they quickly reach states with small accessible state sets and the subsequent exploration is severely restricted. 

\begin{algorithm}[tbp]
    \caption{Resource-Aware Exploration Bonus}
    \label{alg:my_algorithm}
    \begin{algorithmic}[1]
        \STATE {\bfseries Input:} 
        Require the parameter $\beta>0$
        \FOR{$k = 0, 1, 2, \dots$}
        \STATE Take action $a_k$ from the state $s_k$ with the policy $\pi$
        \STATE Receive reward $r_k$ and transit to $s_{k+1}$
        \STATE Calculate the exploration bonus $b(s_k,a_k)$ and the resource-aware coefficient  $g(I(s_{k}))$
        \STATE $r\gets r_k + \beta g(I(s_{k}))b(s_k,a_k) $
        \STATE Update the policy $\pi$ using the reward $r $ 
        \ENDFOR
    \end{algorithmic}
\end{algorithm}

To promote efficient exploration in R3L problems, we propose a novel \textbf{r}esource-\textbf{a}ware \textbf{e}xploration \textbf{b}onus (RAEB), which encourages the agent to explore novel states with large accessible state sets by promoting resource-saving exploration. Specifically, RAEB has the form of

\vspace{-0.4cm}
\begin{align}
\label{eqn:RAEB}
r_i(s,a) = g(I(s)) b(s,a),
\end{align}
where $g:\mathbb{R}^d_{+}\to \mathbb{R}_{+}$ is a non-negative function, namely a resource-aware coefficient, and $b(s,a)$ is the exploration bonus to measure the ``novelty'' of a given state. Intuitively, RAEB encourages the agent to visit unseen states with large accessible state sets. Instantiating RAEB amounts to specifying two design decisions: (1) the measure of ``novelty'', (2) and the resource-aware coefficient. 

\textbf{The measure of novelty}
Surprise-based intrinsic motivation mechanism (Surprise) \cite{Achiam17} has achieved promising performance in hard exploration high-dimensional continuous control tasks, such as SwimmerGather. In this paper, we use surprise---the agent's surprise about its experiences---to measure the novelty of a given state to encourage the agent to explore rarely seen states. We use the KL-divergence of the true transition probability distribution from a transition model which the agent learns concurrently with the policy as the exploration bonus. Thus, the exploration bonus is of the form

\vspace{-0.4cm}
\begin{align}
\label{eqn:prediction_error}
b(s,a) = D_{KL}(p\parallel p_{\phi})(s,a),
\end{align}
where $p:\mathcal{S}\times \mathcal{A}\times \mathcal{S} \to [0, \infty) $ is the true transition probability distribution and $p_\phi$ is the learned transition model. Moreover, researchers \cite{Achiam17} proposed an approximation to the KL-divergence $D_{KL}(p\parallel p_{\phi})(s,a)$,

\vspace{-0.4cm}
\begin{align*}
D_{KL}(p\parallel p_{\phi})(s,a) &= H(p, p_{\phi})(s,a) - H(p)(s,a)\\
&\approx H(p, p_{\phi})(s,a)\\
&= \mathbb{E}_{s^{\prime}\sim p(\cdot|s,a)}[-\log p_{\phi}(s^{\prime}|s,a)]
\end{align*}
\begin{figure*}[t]
    \centering
    \begin{subfigure}{0.82\textwidth}
        \includegraphics[width=\textwidth]{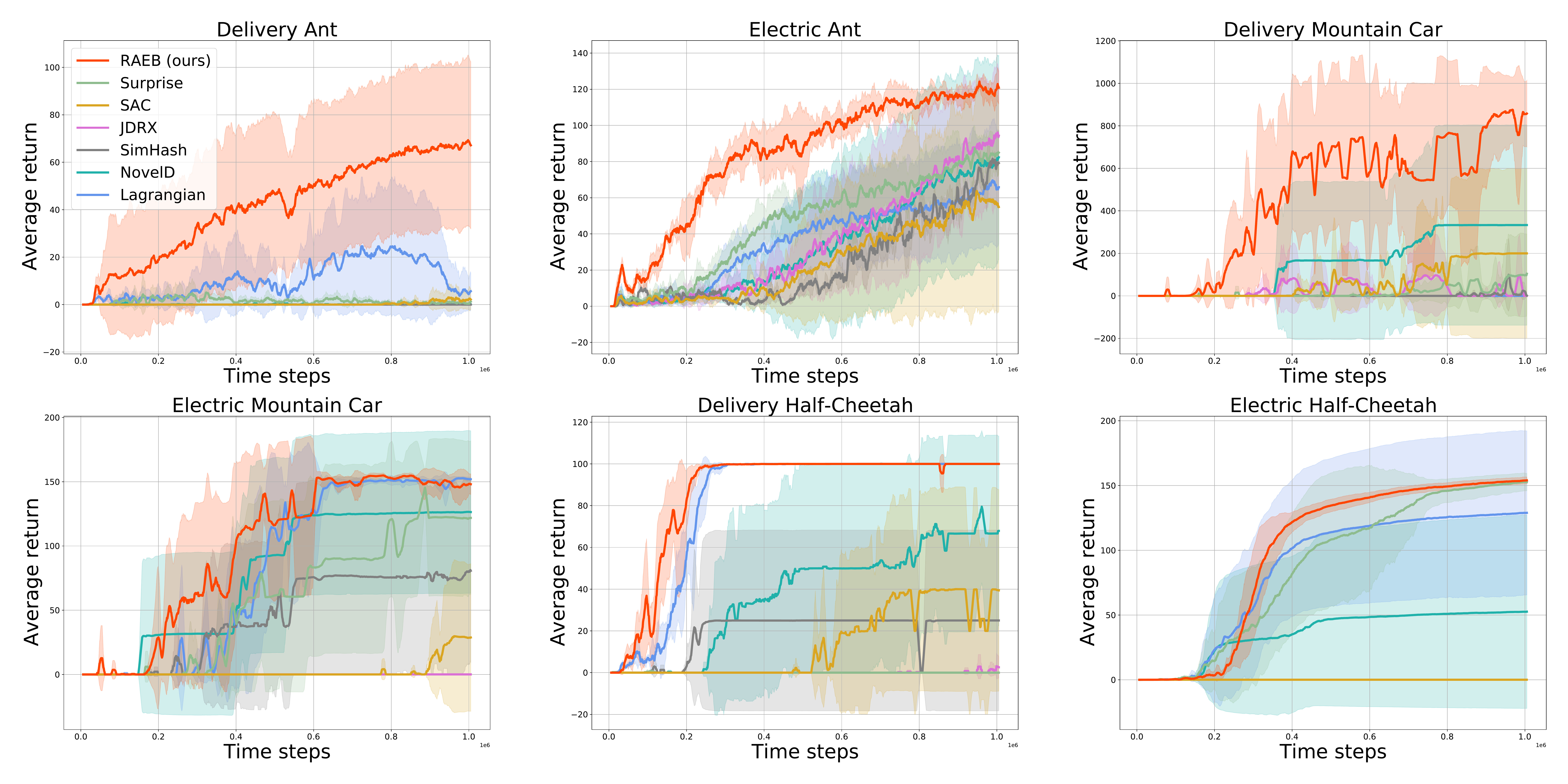}
        
        \vspace{-4mm}
        
        \subcaption{Environments with single resource.}
        \label{fig:evaluation_results}
    \end{subfigure}
    \begin{subfigure}{0.82\textwidth}
        \includegraphics[width=\textwidth]{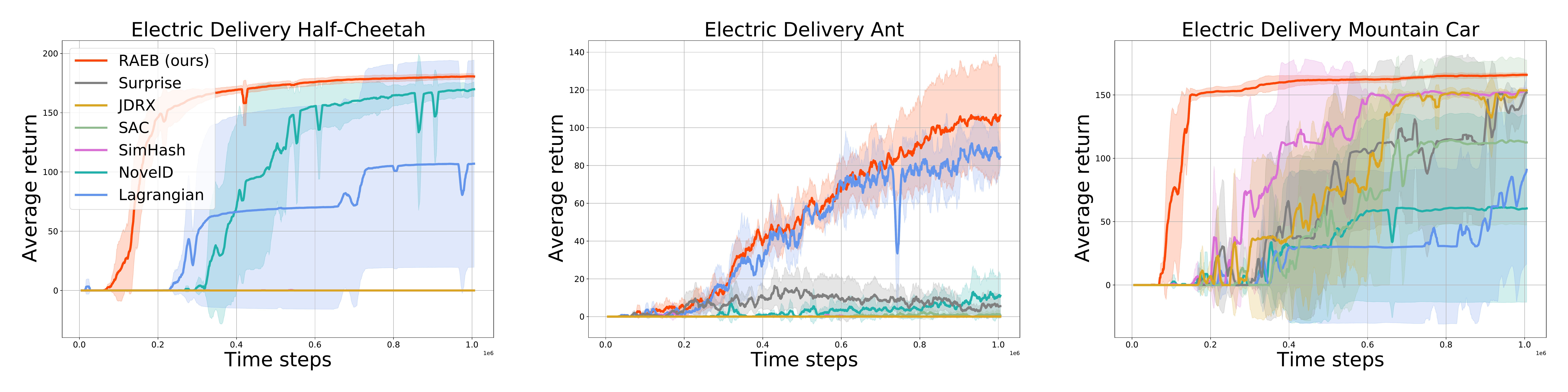}
        
        \vspace{-3mm}
        
        \subcaption{Environments with multiple resources.}
        \label{fig:evaluation_results_multi_resources}
    \end{subfigure}
    \vspace{-2mm}
    \caption{Performance of RAEB and six baselines on nine R3L tasks with limited electricity and/or goods. The solid curves correspond to the mean and the shaded region to the standard deviation over at least five random seeds. For visual clarity, we smooth curves uniformly. The results show that RAEB significantly outperforms these baselines in terms of sample efficiency on several challenging R3L tasks. For Delivery Ant, RAEB improves the sample efficiency by an order of magnitude.}
    \label{fig:evaluation_results_all}
    \vspace{-3mm}
\end{figure*}

\textbf{The resource-aware coefficient}
We define $g:\mathbb{R}^d_{+}\to \mathbb{R}_{+}$ in Eqn.(\ref{eqn:RAEB}) as the resource-aware coefficient. We require the function $g$ being an increasing function of the amount of each type of resources, i.e., for each $i\in [d]$, if two resource vector $\xi,\xi^\prime$ satisfy $\xi_i \leq \xi^\prime_i$ and $\xi_j=\xi^\prime_j$, $j\in [n]\setminus \{i\}$, we have $g(\xi)\leq g(\xi^\prime)$. In this paper, we find that a linear function of the quantity of available resources performs well. 
Suppose the environment has one type of resource, then we define the resource-aware coefficient by

\vspace{-0.4cm}
\begin{align}
\label{eqn:single-resource-aware-coefficient}
g\left(I(s)\right) = \frac{I(s) + \alpha}{I_{\max} +\alpha},
\end{align}
where $\alpha>0$ is a hyperparameter and $I_{\max}$ is the quantity of resources at the beginning of each episode. The hyperparameter $\alpha$ determines the importance of the item $I(s)$ in the resource-aware coefficient and thus controls the degree of resource-saving exploration. 
Similarly, when the environment involves several types of resources and these resources are relatively independent with each other, we can define the resource-aware coefficient by

\vspace{-0.5cm}
\begin{align}
\label{eqn:mul-resource-aware-coefficient}
g\left(I(s)\right) =
\prod_{i=1}^d \frac{I_i(s) + \alpha_i}{I_{\max,i} + \alpha_i}.
\end{align}


\textbf{Discussion about RAEB} We discuss some advantages of RAEB in this part. (1) RAEB makes reasonable usage of resources by promoting novelty-seeking and resource-saving exploration simultaneously. Moreover, RAEB flexibly controls the degree of resource-saving exploration by varying $\alpha$.  
(2) RAEB avoids myopic policies by encouraging the agent to explore novel states with large accessible state sets instead of only exploring novel states.





\textbf{Theoretical results}
We provide some theoretical results of RAEB in this part. 
We restrict the problems to the finite horizon tabular setting and only consider building upon UCB-H \cite{q-learning-with-ucb}.
Researchers \cite{q-learning-with-ucb} 
have shown that $Q$-learning with
UCB-Hoeffding is provably efficient and achieves regret $\widetilde{\mathcal{O}}(\sqrt{H^4 SAT})$. We view an instantiation of RAEB in the tabular setting as Q-learning with a weighted UCB bonus that multiplies the bounded resource-aware coefficient $\beta(s,a)$ and the UCB exploration bonus.
We show that RAEB in the finite horizon tabular setting has a strong theoretical foundation and establishes $\sqrt{T}$ regret.

\begin{theorem}
There exists an absolute constant $c>0$ such that, for any $p\in (0,1)$, choosing $b_t = c\sqrt{H^3\iota /t}$, if there exists a positive real $d$, such that the weights $\beta(s,a) \in [1, d]$ for all state-action pairs $(s,a)$, then with probability $1-p$, the total regret of Q-learning with the weighted UCB bonus (RAEB) establishes $\sqrt{T}$ regret where $\iota :=\log(SAT/p)$.
\label{theorem:convergence}
\end{theorem}
The details of $Q$-learning with the weighted bonus algorithm and the rigorous proof to Theorem \ref{theorem:convergence} are in Appendix A.
As resources are limited in practice, the resource-aware coefficient defined by Eqn. (\ref{eqn:single-resource-aware-coefficient}) is bounded. Theorem \ref{theorem:convergence} shows that RAEB in the tabular setting is provably efficient.

\section{Experiments}

Our experiments have four main goals: (1) Test whether RAEB can significantly outperform state-of-the-art exploration strategies and constrained reinforcement learning methods in R3L tasks. (2) Analyze the effect of each component in RAEB and the sensitivity of RAEB to hyperparameters.
(3) Evaluate RAEB variants to provide further insight into RAEB. (4) Illustrate the exploration of RAEB. 


As mentioned in Section \ref{sec:inefficient_exploration}, we design control tasks with limited electricity and/or goods based on Gym \cite{gym} and Mujoco \cite{mujoco}. The agent in these tasks can be a continuous mountain car, a 2D robot (called a `half-cheetah'), and a quadruped robot (called an `ant'). Then we call these tasks Electric Mountain Car, Electric Half-Cheetah, Electric Ant,  Delivery Mountain Car,  Delivery Half-Cheetah, Delivery Ant, Electric Delivery Mountain Car, Electric Delivery Half-Cheetah, and Electric Delivery Ant, respectively. For all environments, We augment the states with the available resources.
In electric tasks, the agent aims to reach a specific unknown destination with limited electricity. In delivery tasks, the agent aims to deliver the goods to a specific unknown destination. In electric delivery tasks, the agent aims to deliver the goods to a specific unknown destination with limited electricity. Please refer to Appendix B.1 for more details about the environments. 


We find that the Surprise
can efficiently reach the unknown destination in tasks without resource restriction. (Please refer to Appendix C for detailed results.) Therefore, we implement RAEB on top of the Surprise method. Furthermore, we use soft actor critic (SAC) \cite{pmlr-v80-haarnoja18b}, the state-of-the-art off-policy method,
as our base reinforcement learning algorithm. The details of the experimental setup are in Appendix B.2. 
Moreover, we project unavailable actions that agents output onto the available action set (see Appendix D) to ensure any policy trained in the experiments belongs to the feasible policy set $\Pi$. 

\subsection{Evaluation and Comparison Analysis}\label{sec:experiment_evaluation}

\textbf{Baselines}
    We compare RAEB with several exploration strategies commonly used to solve problems with high-dimensional continuous state-action spaces, such as Ant Maze \cite{max}. The baselines include state-of-the-art exploration strategies and a state-of-the-art constrained reinforcement learning method. For exploration strategies, we compare to Surprise \cite{Achiam17}, a state-of-the-art prediction error-based exploration method, NovelD \cite{noveld}, a state-of-the-art random network based exploration method, Jensen-R\'enyi Divergence Reactive Exploration (JDRX) 
    \cite{max}, a state-of-the-art information gain based exploration method, SimHash \cite{simhash}, a state-of-the-art count-based exploration method, and soft actor critic (SAC) \cite{pmlr-v80-haarnoja18b}, a state-of-the-art off-policy algorithm. In terms of constrained reinforcement learning methods, we compare to the Lagrangian method, which uses the consumed resources at each step as a penalty and has achieved strong performance in constrained reinforcement learning tasks \cite{saferl_benchmark}. Notice that we focus on the model-free reinforcement learning settings in this paper.

\noindent \textbf{Evaluation}
    We first compare RAEB to the aforementioned baselines in R3L tasks with a single resource, i.e.,  delivery tasks and tasks with limited electricity. We then compare RAEB to the aforementioned baselines in R3L tasks with multiple resources, i.e., delivery tasks with limited electricity. For all environments, we use the intrinsic reward coefficient $\beta=0.25$. We use $\alpha=0.25 I_{\max}$ for delivery tasks, $\alpha=2.5 I_{\max}$ for tasks with limited electricity, and $\boldsymbol{\alpha}=\left[0.25 I_{\max}, 2.5 I_{\max}\right]$ for delivery tasks with limited electricity. Note that $\alpha$ controls the degree of resource-saving exploration, and we need to adjust $\alpha$ according to the scarcity of resources in practice. 

    First, Figure \ref{fig:evaluation_results} shows that RAEB significantly outperforms the baselines on several challenging R3L tasks with a single resource. For Delivery Ant, we show that the average return of the baselines after $1\times 10^6$ steps is at most $5.65$ while RAEB achieves similar performance ($11.13$) using only $1\times 10^5$ steps. Please refer to Appendix C for detailed results. This result demonstrates that RAEB improves the sample efficiency by an order of magnitude on Delivery Ant. Moreover, some baselines even struggle to attain scores better than random agents on several R3L tasks, which demonstrates that efficient exploration is extremely challenging for some baselines in R3L tasks. Compared to the Lagrangian method, RAEB still achieves outstanding performance on several challenging R3L tasks. The major reason is that, the constrained reinforcement learning method aims to find policies that satisfy resource constraints, and struggles to make reasonable usage of resources in R3L tasks.
    
    Second, Figure \ref{fig:evaluation_results_multi_resources} shows that RAEB significantly outperforms the baselines on several R3L tasks with two types of resources. The results highlight the potential of RAEB for efficient exploration in R3L tasks with multiple resources. RAEB scales well in R3L tasks with multiple resources as the two types of resources are independent, and each resource's resource-aware coefficient encourages the corresponding resource-saving exploration.

\begin{figure}[t]
    \centering
    \includegraphics[width=0.48\columnwidth]{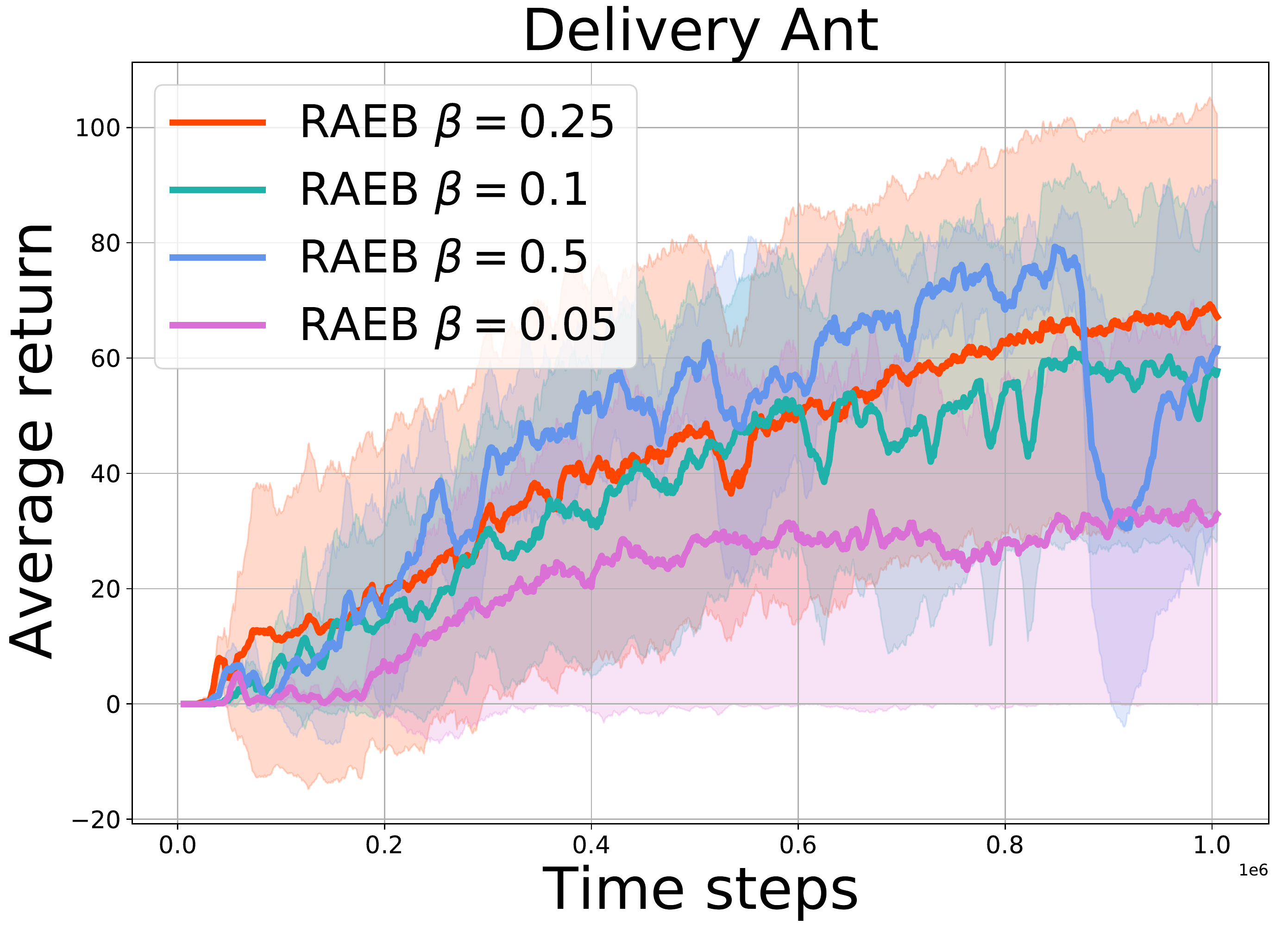}
    \includegraphics[width=0.48\columnwidth]{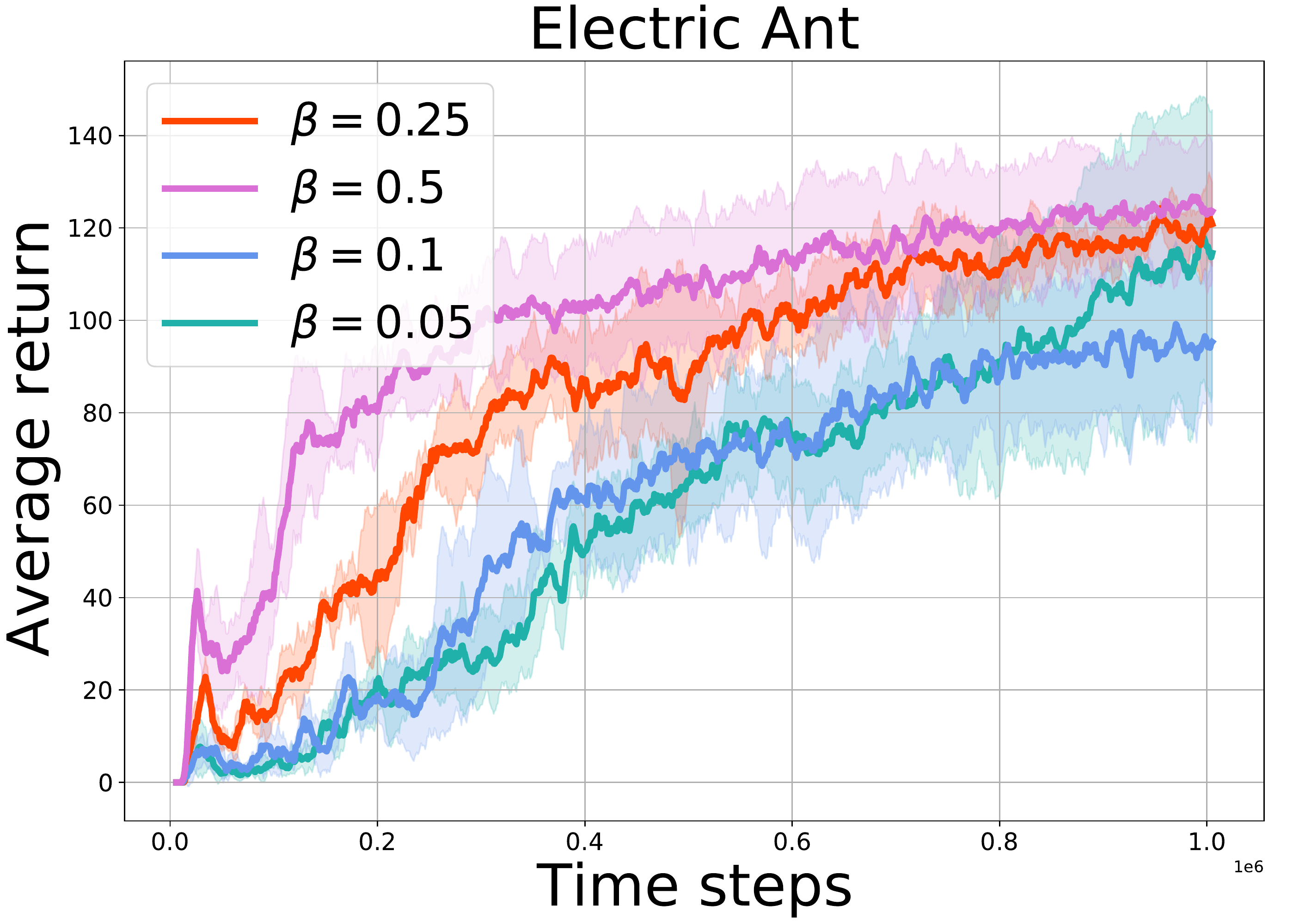}
    \vspace{-2mm}
    \caption{Learning curves on Delivery Ant and Electric Ant. RAEB achieves similar average performance with different reward weighting coefficients $\beta$, which shows that RAEB is insensitive to $\beta$.}
    \label{fig:beta_sensitivity}
    \vspace{-4mm}
\end{figure}

\begin{figure}[t]
    \centering
    \includegraphics[width=0.48\columnwidth]{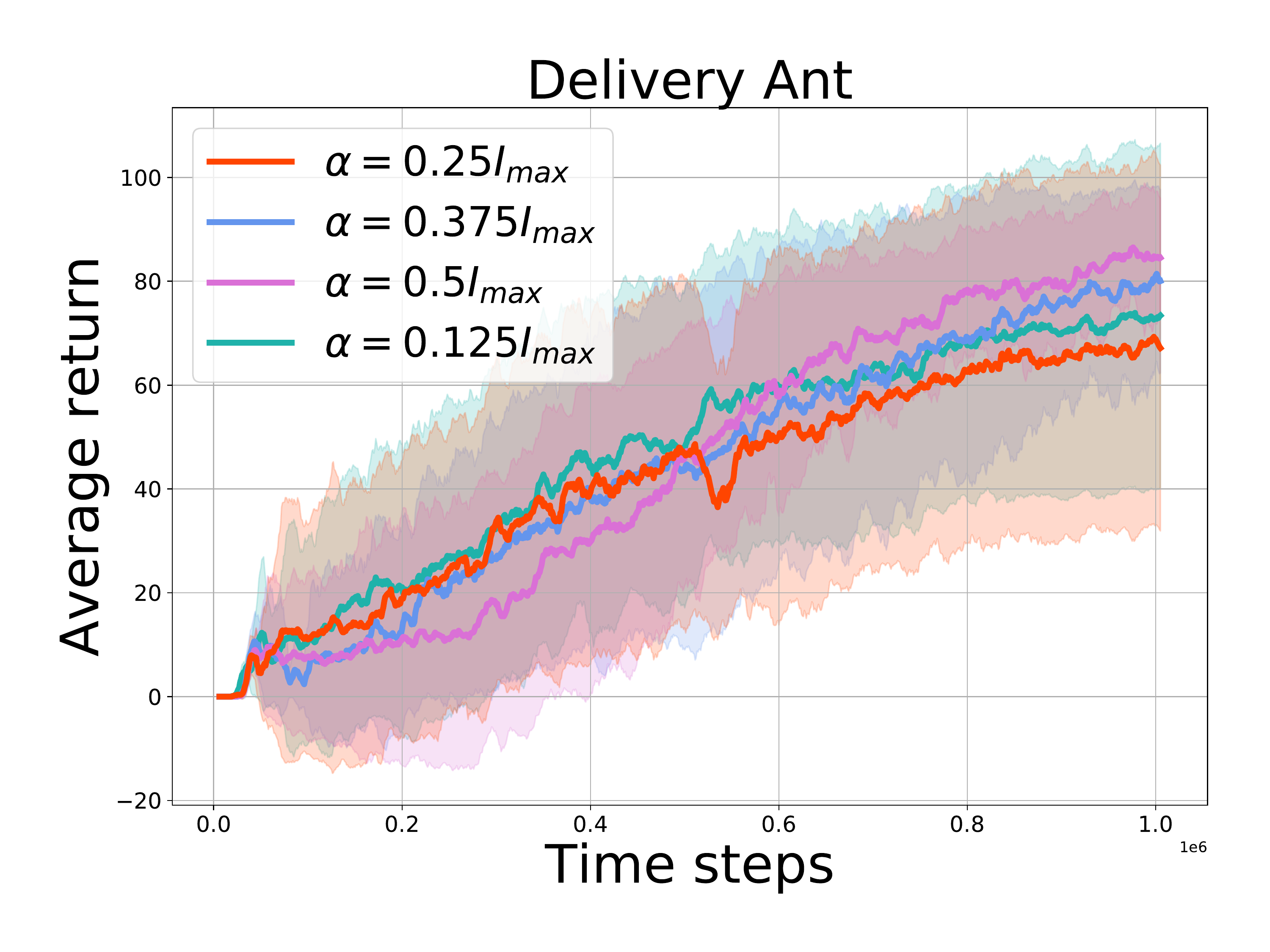}
    \includegraphics[width=0.48\columnwidth]{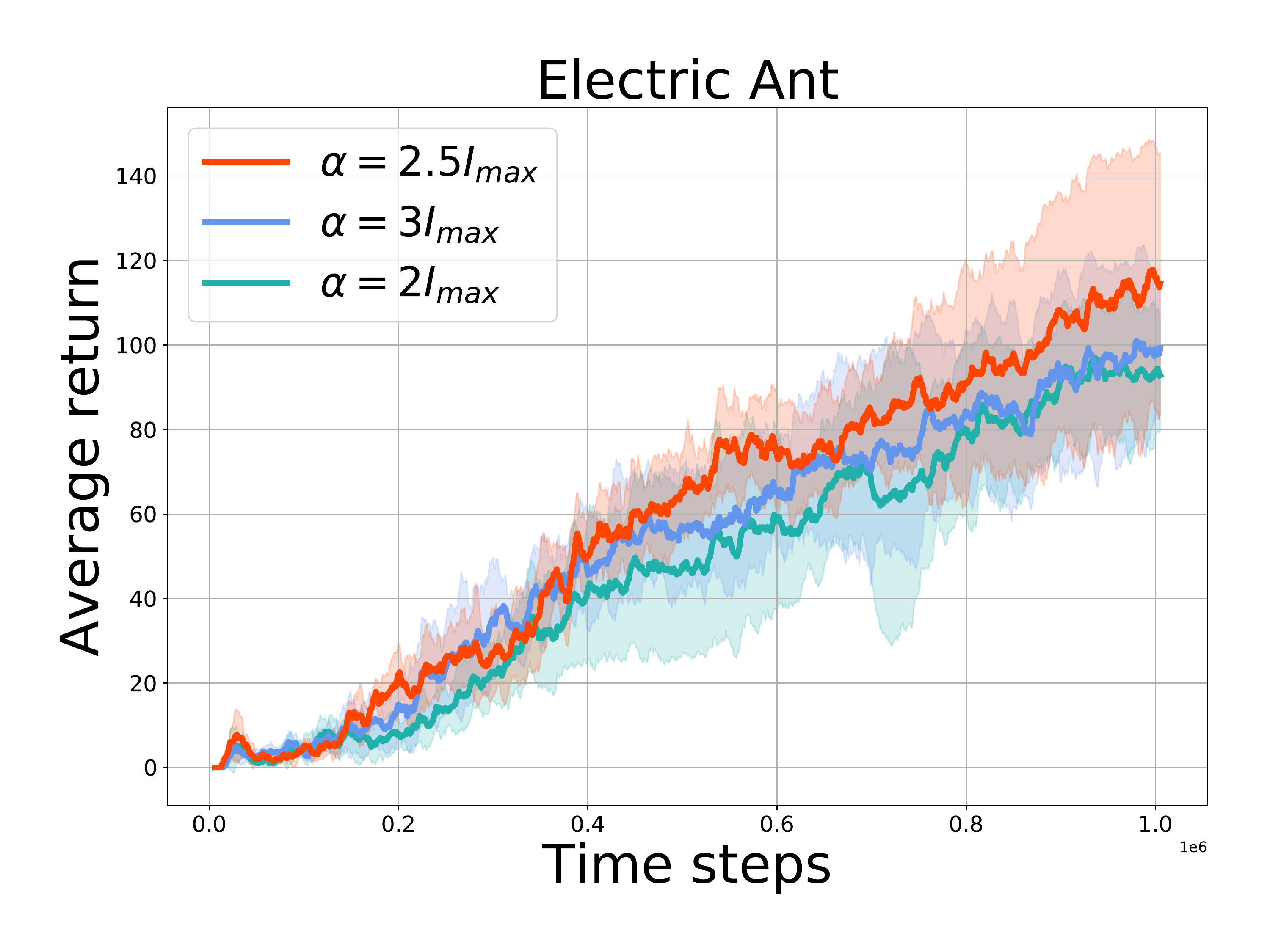}
    \vspace{-2mm}
    \caption{Learning curves on Delivery Ant and Electric Ant. RAEB achieves similar average performance with different $\alpha$, which demonstrates that RAEB is insensitive to $\alpha$.}
    \label{fig:alpha_sensitivity}
\end{figure}

\subsection{Ablation Study}

First, we analyze the sensitivity of RAEB to hyperparameters $\beta$ and $\alpha$. Then, We perform ablation studies to understand the contribution of each individual component in RAEB.
In this section, we conduct experiments on Delivery Ant and Electric Ant, which have high resolution power due to their difficulty. All results are reported over at least four random seeds. For completeness, we provide additional results in Appendix C.  

\noindent \textbf{Sensitivity analysis}
First, we compare the performance of RAEB with different reward weighting coefficients $\beta$. Intuitively, higher $\beta$ values should cause more exploration, while too low of an $\beta$ value reduces RAEB to the base algorithm, i.e., SAC. In the experiments, we set the reward weighting coefficient $\beta=0.5, 0.25, 0.1,$ and $0.05$, respectively. Though choosing a proper reward  weighting coefficient remains an open problem for existing intrinsic-reward-based exploration methods \cite{large_curiosity}, the results in Figure \ref{fig:beta_sensitivity} show that there is a wide $\beta$ range for which RAEB achieves comparable average performance on Delivery Ant and Electric Ant. Then, we compare the performance of RAEB with different $\alpha$ on Delivery Ant and Electric Ant. The results in Figure \ref{fig:alpha_sensitivity} show that
there is a wide $\alpha$ range for which RAEB achieves competitive performance on Delivery Ant and Electric Ant. Overall, the results in Figures \ref{fig:beta_sensitivity} and \ref{fig:alpha_sensitivity} show that RAEB is insensitive to hyperparameters $\beta$ and $\alpha$. 

\noindent\textbf{Contribution of each component}
RAEB contains two components: the measure of novelty and the resource-aware coefficient. We perform ablation studies to understand the contribution of each individual component. Figure \ref{fig:component} shows that the full algorithm RAEB significantly outperforms each ablation with only a single component on Delivery Ant and Electric Ant. On delivery Ant, RAEB without resources and RAEB without surprise even struggle to perform better than a random agent. The results show that each component is significant for efficient exploration in these R3L tasks. 
Though the surprise bonus encourages novelty-seeking exploration, it tends to exhaust resources fast. And only the resource-aware coefficient tends to encourage the agent to save resources without consuming resources, while resource-consuming exploration is significant in R3L tasks.  


\begin{figure}[t]
    \centering
    \includegraphics[width=0.48\columnwidth]{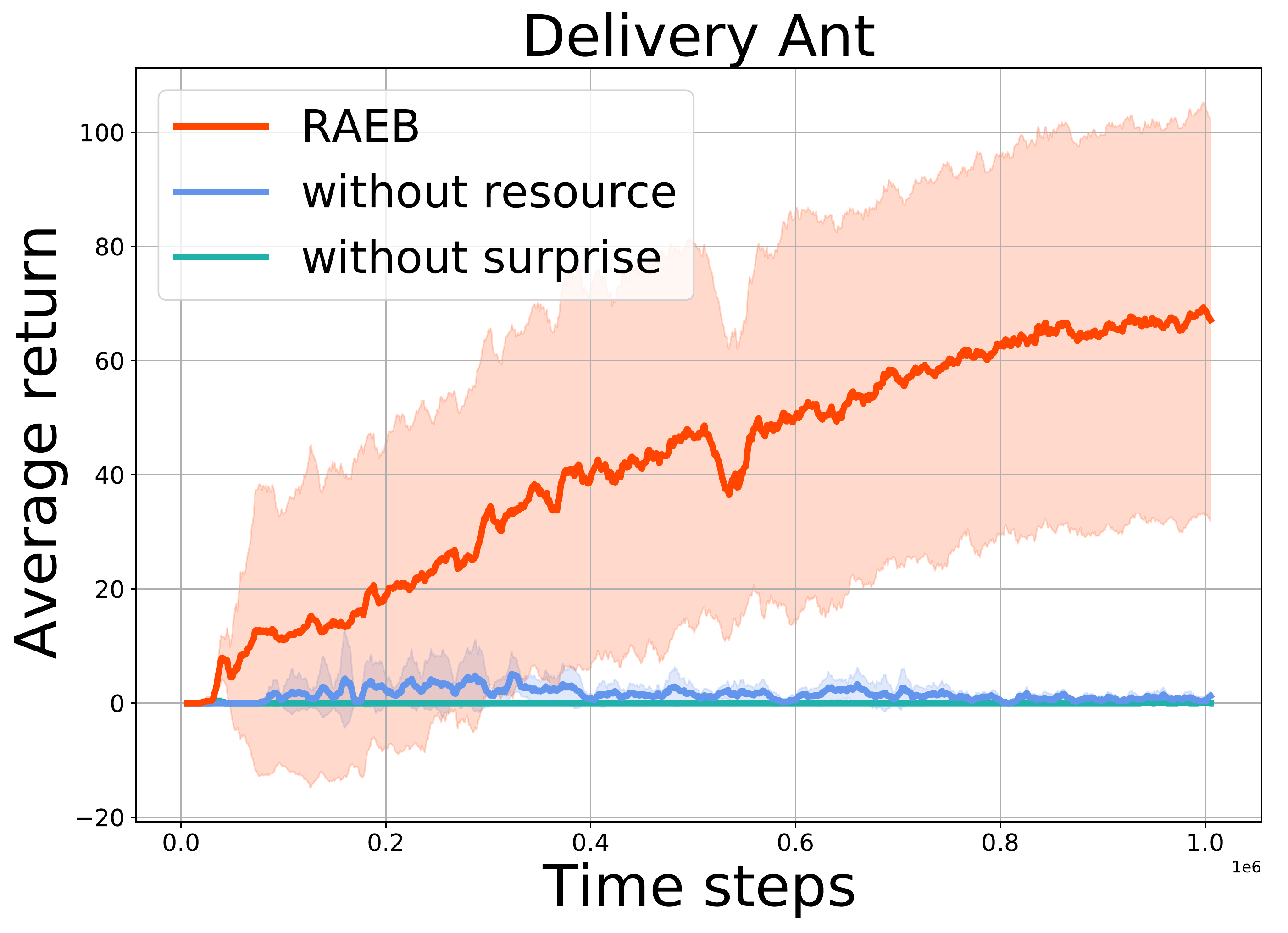}
    \includegraphics[width=0.48\columnwidth]{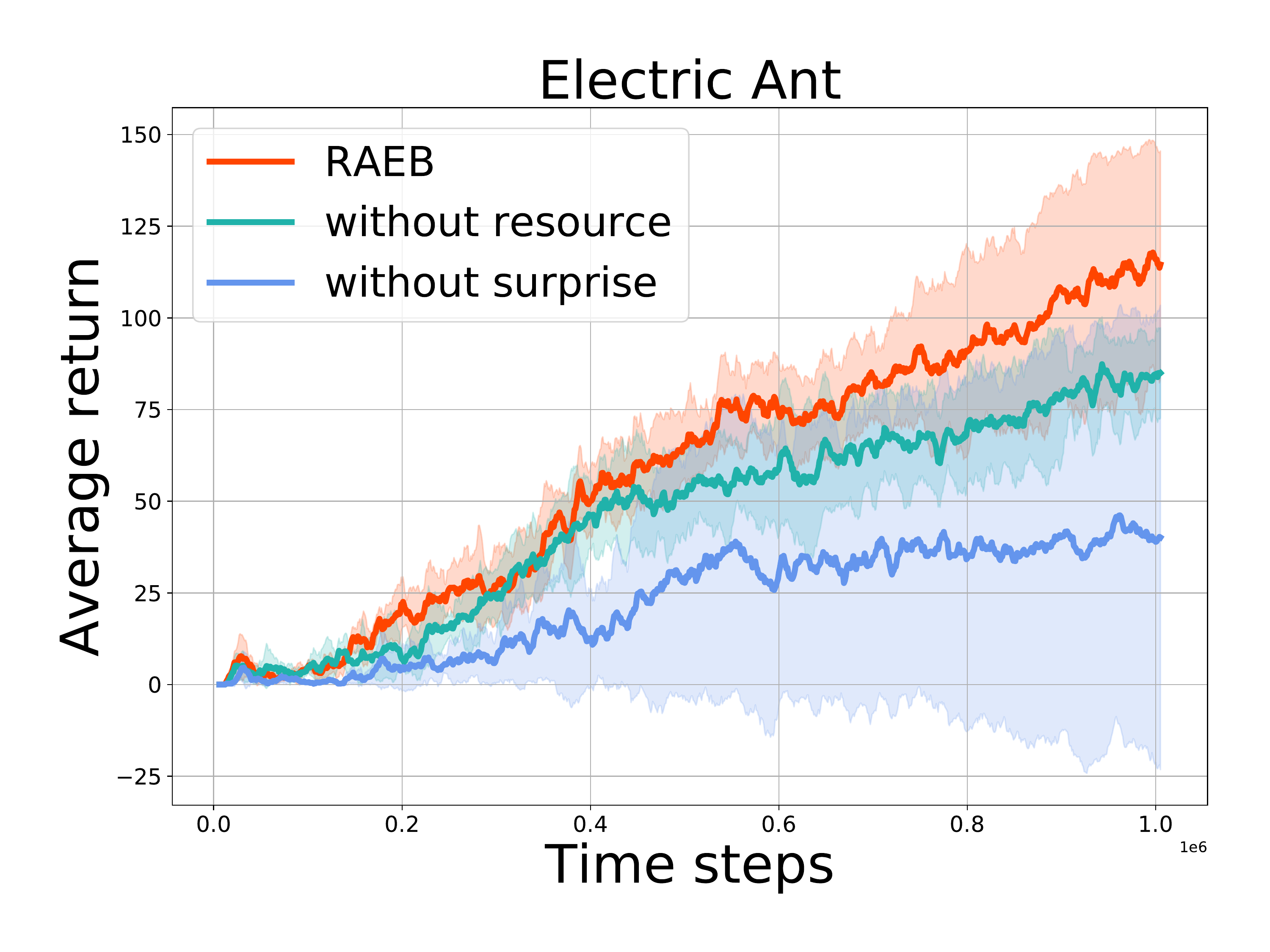}
    \vspace{-2mm}
    \caption{Learning curves on Delivery Ant and Electric Ant. The results show that each component of RAEB is significant for efficient exploration in R3L tasks.}
    \label{fig:component}
    \vspace{-3mm}
\end{figure}

\begin{figure}[t]
    \centering
    \includegraphics[width=0.48\columnwidth]{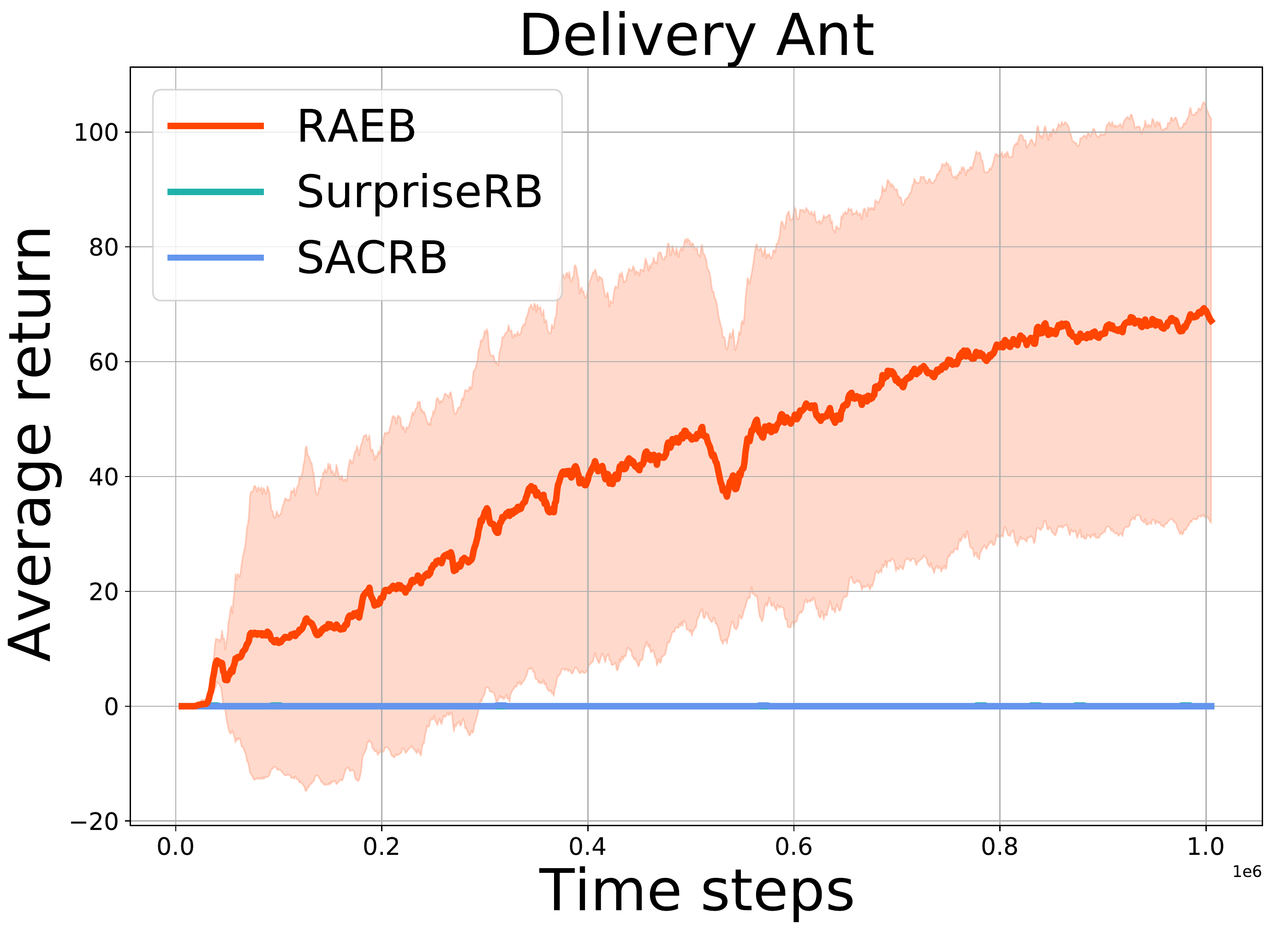}
    \includegraphics[width=0.48\columnwidth]{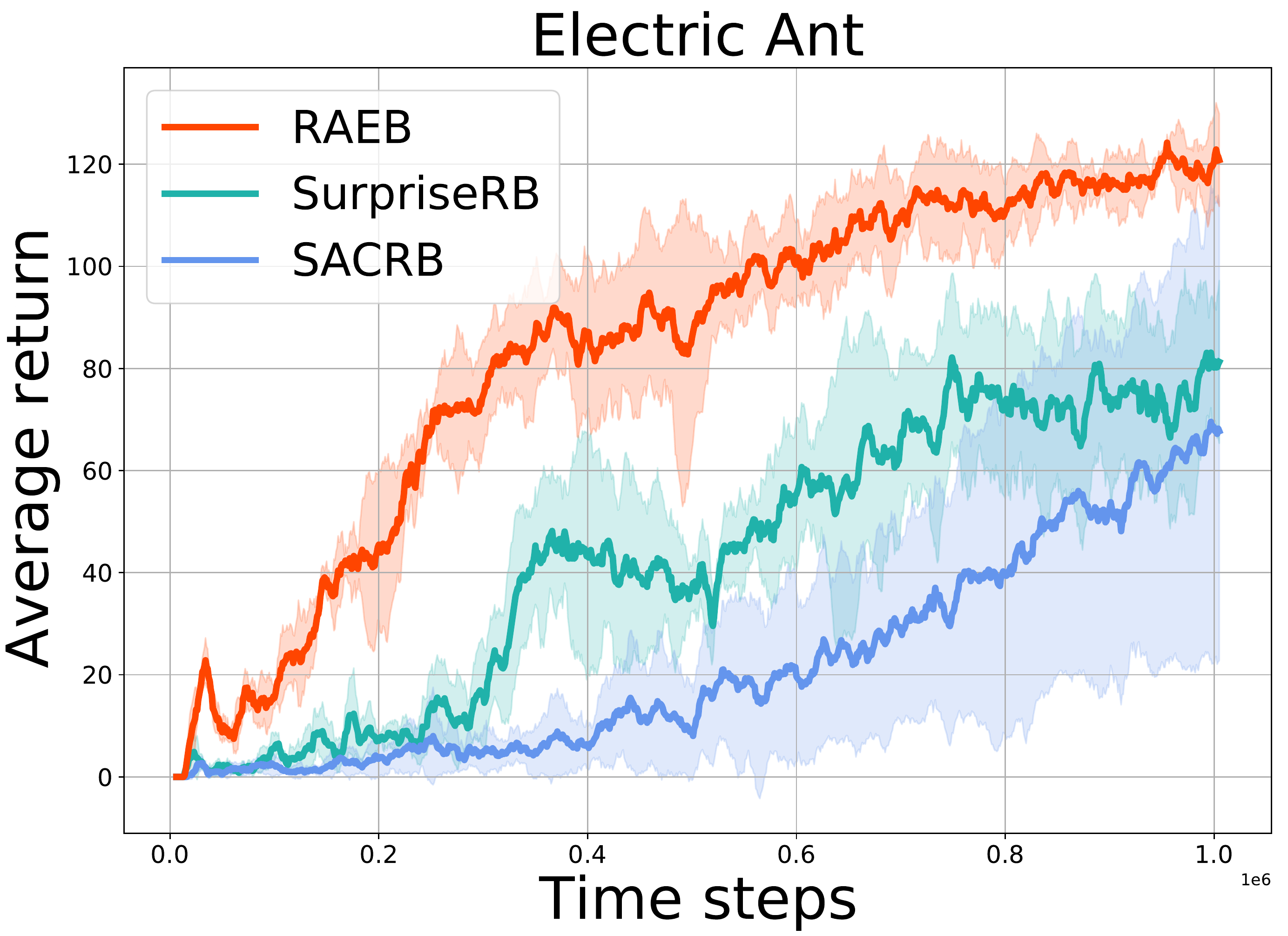}
    \vspace{-3mm}
    \caption{We compare RAEB to two RAEB variants on Delivery Ant and Electric Ant. The results show that these variants struggle to achieve outstanding performance in both environments compared to RAEB.}
    \label{fig:ablation_trivial_methods}
\end{figure}

\begin{figure}[t]
    \centering
    \includegraphics[width=0.48\columnwidth]{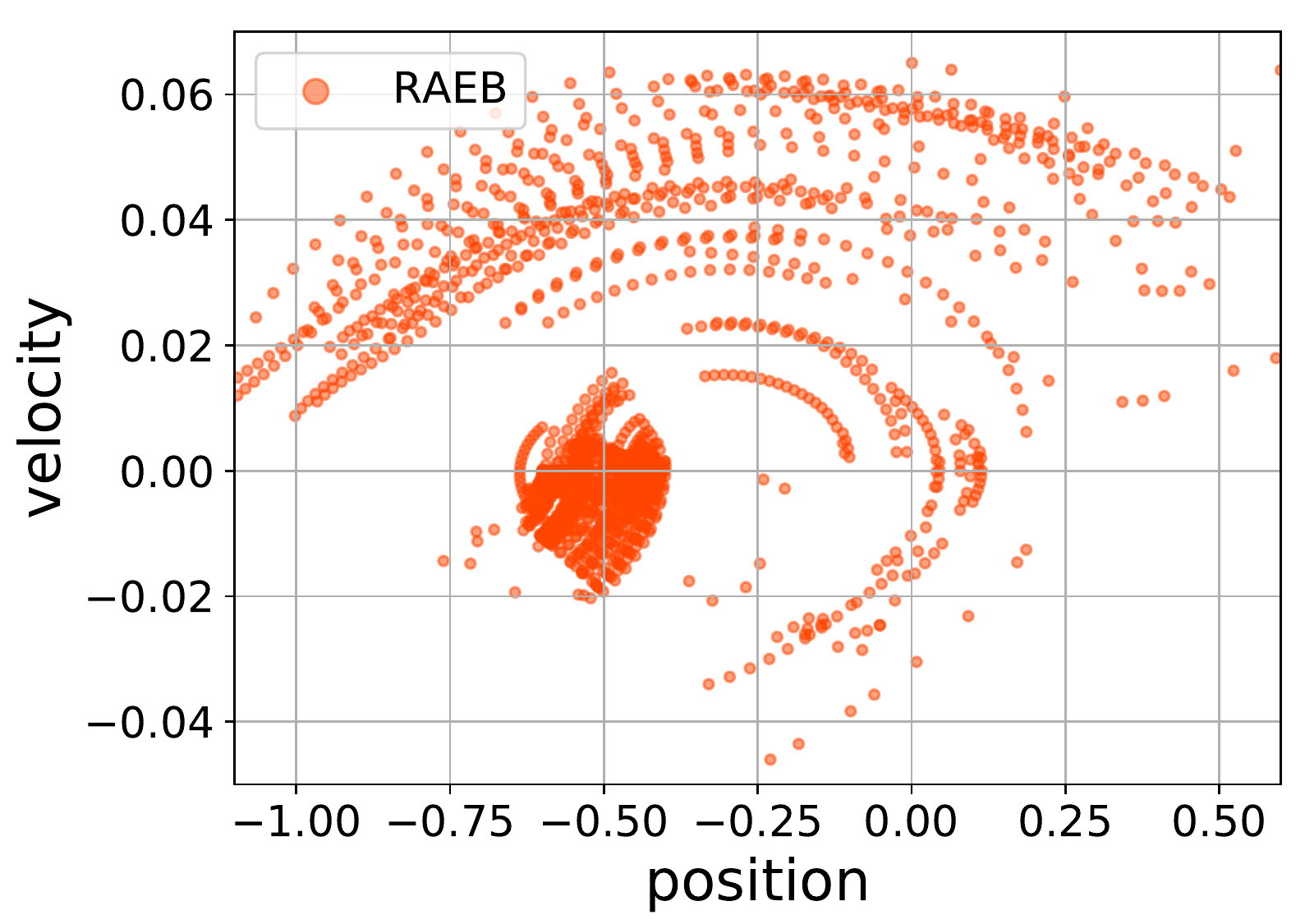}
    \includegraphics[width=0.48\columnwidth]{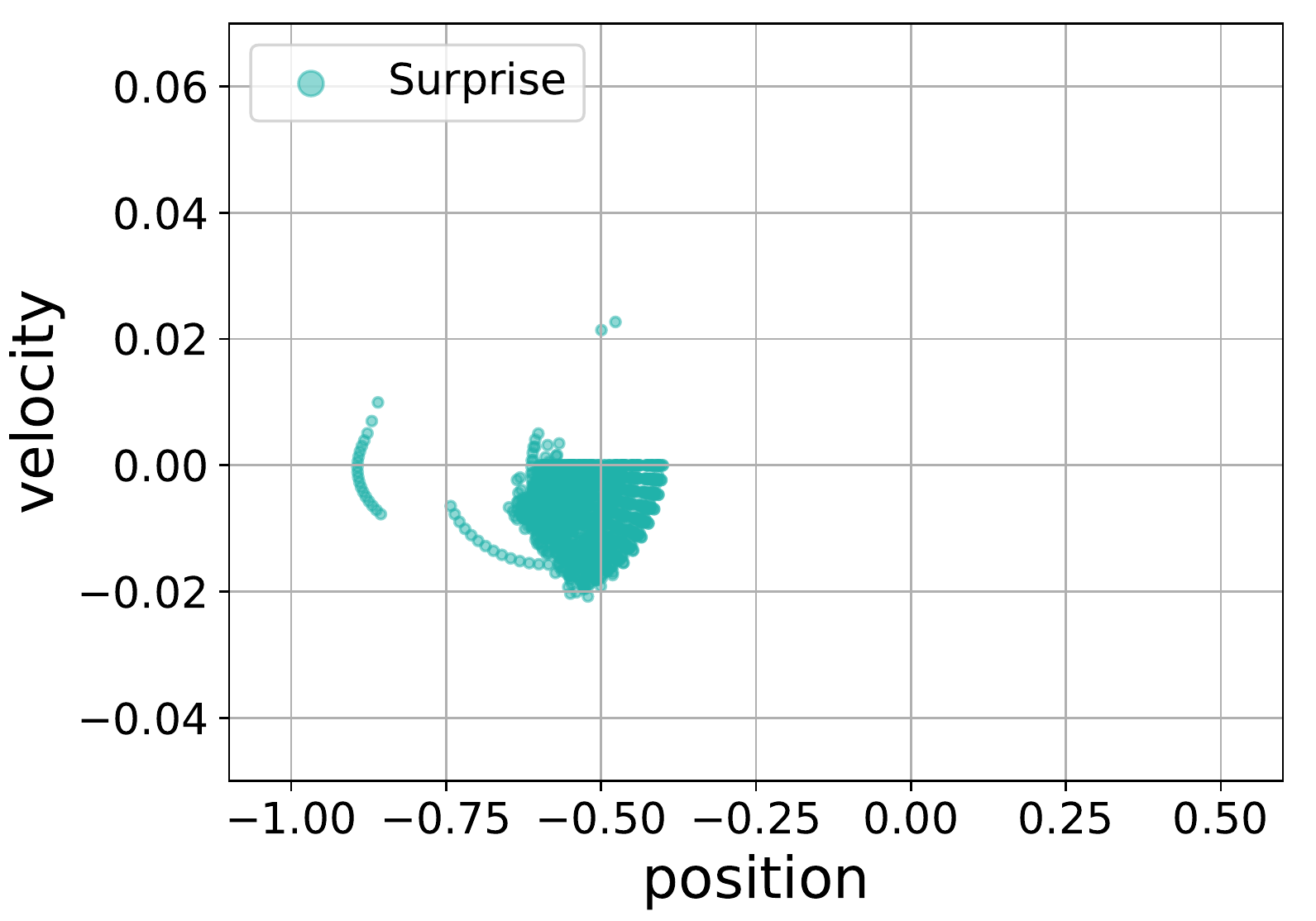}
    \includegraphics[width=0.48\columnwidth]{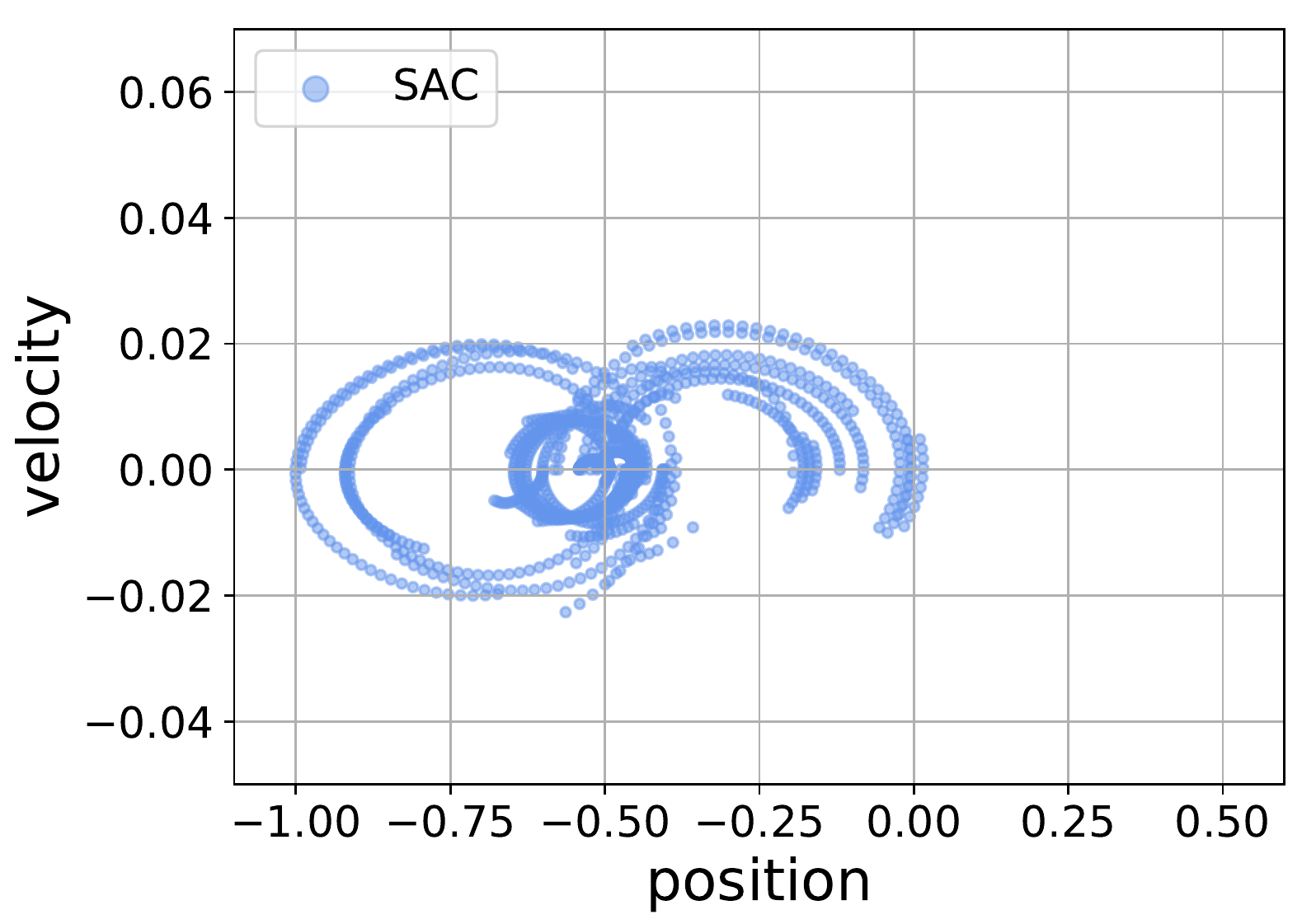}
    \includegraphics[width=0.48\columnwidth]{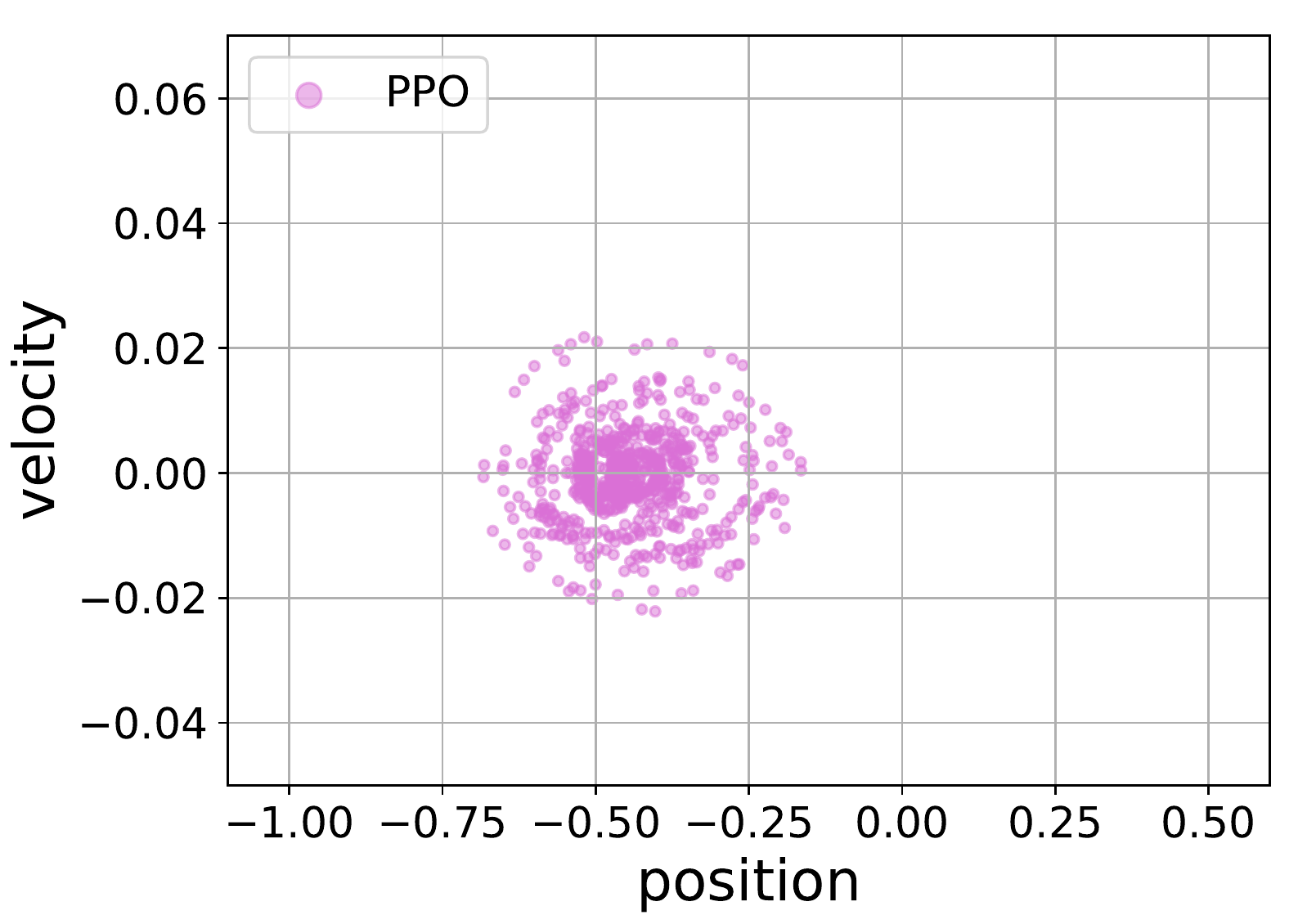}
    \vspace{-2mm}
    \caption{Illustration of RAEB exploration on Delivery Mountain Car. Colored points illustrate the states (position and velocity of the car) where the agent consumes resources in 0.2 million steps.}
    \label{fig:visualization}
\end{figure}

\subsection{Compared to RAEB variants}

Based on the idea of preventing the agent from exhausting resources fast, it is natural to use the quantity of the remaining resources as a reward bonus instead of a resource-aware coefficient. We use the bonus to train SAC and Surprise, called SAC with resources bonus (SACRB) and Surprise with resources bonus (SurpriseRB), respectively. Specifically, SACRB and SurpriseRB add an additional reward bonus $c \frac{I(s)+\alpha}{I_{\max}+\alpha}$ to the reward. We compare RAEB to these variants on Delivery Ant and Electric Ant. Figure \ref{fig:ablation_trivial_methods} demonstrates that RAEB significantly outperforms these variants in both two environments, which shows the superiority of RAEB compared to its variants. For Delivery Ant, the RAEB variants even struggle to perform better than random agents. The major advantage of RAEB over SurpriseRB is that it is able to well combine the exploration ability of the surprise bonus and the resource-aware bonus by using a resource-aware coefficient instead of adding an additional bonus, as the scale between the surprise and the resource-aware bonus can be extremely different. Moreover, SurpriseRB is sensitive to the hyperparameter $c$. (See Appendix C)



\subsection{Illustration of RAEB exploration}

To illustrate the exploration in the Delivery Mountain Car environment, we visualize the states where the agent unloads the goods in Figure \ref{fig:visualization}. Notice that we visualize the position and velocity of the car while ignoring the quantity of remaining resources for conciseness. Figure \ref{fig:visualization} shows that the agents trained by PPO, SAC, and Surprise unload the goods in the same states repeatedly, leading to inefficient exploration. Moreover, Figure \ref{fig:visualization} shows that RAEB significantly reduces unnecessary resource-consuming trials, while effectively encouraging the agent to explore unvisited states. 

\section{Conclusion}

In this paper, we first formalize the decision-making tasks with resources as a resource-restricted reinforcement learning, and then propose a novel resource-aware exploration bonus to make reasonable usage of resources. Specifically, we quantify the RAEB of a given state by both the measure of novelty and the quantity of available resources. We conduct extensive experiments to demonstrate that the proposed RAEB significantly outperforms state-of-the-art exploration strategies on several challenging robotic delivery and autonomous electric robot tasks.

\section{Acknowledgments}
We would like to thank all the anonymous reviewers for their insightful comments. This work was supported in part by National Science Foundations of China grants U19B2026, U19B2044, 61836006, and 62021001, and the Fundamental Research Funds for the Central Universities grant WK3490000004.

\bibliography{aaai23}

\newpage
\newpage

\appendix

\section{Theoretical Analysis}
\subsection{Proof of Theorem \ref{theorem:convergence}}

Our analysis is based on $Q$-learning with UCB-Hoeffding \cite{q-learning-with-ucb}. In this subsection, we first introduce the background and then show the proof of Theorem \ref{theorem:convergence}.

\textbf{Background} We define a tabular episodic Markov decision process by a tuple $(\mathcal{S}, \mathcal{A}, H, \mathbb{P}, \mathrm{r}),$ where $\mathcal{S}$ is the state space with $|\mathcal{S}|=S, \mathcal{A}$ is the action space with $|\mathcal{A}|=A, H$ is the number of steps in each episode, $\mathbb{P}_{h}(\cdot | s, a)$ is the probability distribution over states if the agent selects action $a$ in state $s$ at step $h \in[H],$ and $r_{h}: \mathcal{S} \times \mathcal{A} \rightarrow[0,1]$ is the deterministic reward function at step $h .$ We denote by  $N_h(s,a)$ the visiting count of the state-action pair $(s,a)$ at step $h$. The Bellman equation and the Bellman optimality equation are
\begin{align}
\left\{\begin{array}{l}
V_{h}^{\pi}(s)=Q_{h}^{\pi}\left(s, \pi_{h}(s)\right) \\
Q_{h}^{\pi}(s, a)=\left(r_{h}+\mathbb{P}_{h} V_{h+1}^{\pi}\right)(s, a)\\
V_{H+1}^{\pi}(s)=0 \quad \forall s \in \mathcal{S}
\end{array}\right.
\end{align}
and
\begin{align}
\left\{\begin{array}{l}
V_{h}^{\star}(s)=\max _{a \in \mathcal{A}} Q_{h}^{\star}(s, a) \\
Q_{h}^{\star}(s, a)=\left(r_{h}+\mathbb{P}_{h} V_{h+1}^{\star}\right)(s, a) \\
V_{H+1}^{\star}(s)=0 \quad \forall s \in \mathcal{S}\quad ,
\end{array}\right. 
\end{align}
where we define $ \left[\mathbb{P}_{h} V_{h+1}\right](s, a) \dot= \mathbb{E}_{s^{\prime} \sim \mathbb{P}(\cdot | s, a)} V_{h+1}\left(s^{\prime}\right)$.

The agent interacts with the environment for $K$ episodes $k=1,2, \ldots, K,$ and we arbitrarily pick a starting state $s_{1}^{k}$ for each episode $k,$ and the policy before starting the $k$ -th episode is $\pi_{k}$. The total regret is
$$\operatorname{Regret}(K)=\sum_{k=1}^{K}\left( V_{1}^{*}\left(s_{1}^{k}\right)-V_{1}^{\pi^{k}}\left(s_{1}^{k}\right)\right) .$$


The algorithm $Q$-learning with UCB-Hoeffding can be found in \cite{q-learning-with-ucb}. The update rule is 
\begin{align*}
Q_{h}\left(s_{h}, a_{h}\right) \leftarrow &\left(1-\alpha_{t}\right) Q_{h}\left(s_{h}, a_{h}\right)\\
&+\alpha_{t}\left[r_{h}\left(s_{h}, a_{h}\right)+V_{h+1}\left(s_{h+1}\right)+b_{t}\right],
\end{align*}
where $t = N_h(s_h, a_h)$, $\alpha_t = \frac{H+1}{H+t}$, and $b_t$ is the exploration bonus.

The difference between our algorithm $Q$-learning with weighted bonus and $Q$-learning with UCB-Hoeffding is the update rule. The update rule of  $Q$-learning with weighted bonus is 
\begin{align*}
Q_{h}\left(s_{h}, a_{h}\right) &\leftarrow \left(1-\alpha_{t}\right) Q_{h}\left(s_{h}, a_{h}\right)\\
&+\alpha_{t}\left[r_{h}\left(s_{h}, a_{h}\right)+V_{h+1}\left(s_{h+1}\right)+\beta(s_h,a_h)b_{t}\right],
\end{align*}
where $\beta(s_h,a_h)\in [1,d] $ is the coefficient of the exploration bonus.


Since $Q$-learning with weighted bonus additionally introduces a bounded weight $\beta(s,a)$ to the UCB exploration bonus, we slightly adapt the Lemma 4.2 in the paper \cite{q-learning-with-ucb}.
\begin{lemma}(Adapted slightly from \cite{q-learning-with-ucb})\\
There exists an absolute constant $c>0$ such that, for any $p \in(0,1),$ letting $b_{t}=c \sqrt{H^{3} \iota / t},$ if $\beta(s,a) \in [1, d]$ , we have $w_{t}=2 \sum_{i=1}^{t} \alpha_{t}^{i} \beta_i b_{i} \leq 4 cd \sqrt{H^{3} \iota / t}$ and, with probability at
least $1-p,$ the following holds simultaneously for all $(s, a, h, k) \in \mathcal{S} \times \mathcal{A} \times[H] \times[K]:$
\begin{align*}
0 & \leq\left(Q_{h}^{k}-Q_{h}^{\star}\right)(s, a) \\
&\leq \alpha_{t}^{0} H+\sum_{i=1}^{t} \alpha_{t}^{i}\left(V_{h+1}^{k_{i}}-V_{h+1}^{\star}\right)\left(s_{h+1}^{k_{i}}\right)+w_{t},
\end{align*}
where $t=N_{h}^{k}(s, a)$ and $k_{1}, \ldots, k_{t}<k$ are the episodes where $(s, a)$ was taken at step $h .$
\end{lemma}
\begin{proof}
The proof is the same as \cite{q-learning-with-ucb}), except for $w_{t}=2 \sum_{i=1}^{t} \alpha_{t}^{i} \beta_i b_{i}\leq 2d \sum_{i=1}^{t} \alpha_{t}^{i} b_{i} \leq 4 c d\sqrt{H^{3} \iota / t}$ .
\end{proof}

\begin{theorem}
There exists an absolute constant $c>0$ such that, for any $p\in (0,1)$, choosing $b_t = c\sqrt{H^3\iota /t}$, if there exists a positive real $d$, such that the weights $\beta(s,a) \in [1, d]$ for all state-action pairs $(s,a)$, then with probability $1-p$, the total regret of Q-learning with weighted bonus is at most $O(\sqrt{dH^4SAT\iota})$, where $\iota :=\log(SAT/p)$.
\label{theorem:convergence}
\end{theorem}
\begin{proof}
From the proof in \cite{q-learning-with-ucb}, we have
\begin{align*}
\textup{Regret}(K) \leq O\left(H^{2} S A+cH\sqrt{T\iota} +\sum_{h=1}^{H} \sum_{k=1}^{K}\left(\beta_{n_{h}^{k}}\right)\right),
\end{align*}
 where 
$n_{h}^k = N_h(s_h^k,a_h^k)$.

Considering $\sum_{k=1}^{K}\beta_{n_{h}^{k}}$, we have
\begin{align*}
 \sum_{k=1}^{K}\beta_{n_{h}^{k}}
\leq &  \sum_{k=1}^{K} 4cd \sqrt{\frac{H^3\iota}{n_h^k}}\\
\leq & \sum_{s,a} \sum_{n=1}^{N_{h}^K(s,a)} 4cd \sqrt{\frac{H^3\iota}{n}}\\
\overset{b}{\leq} &\sum_{s,a} \sum_{n=1}^{\frac{K}{SA}}4cd \sqrt{\frac{H^3\iota}{n}}\\
\leq & \sum_{s,a}\left( \int_{1}^\frac{K}{SA}4cd \sqrt{\frac{H^3\iota}{x}}dx\right)\\
\leq & O(\sqrt{dH^2SAT\iota}),
\end{align*}
where in $\overset{b}{\leq}$, we use the fact that $\sum_{s,a}\sum_{n=1}^{N_{h}^K(s,a)} 4cd \sqrt{\frac{H^3\iota}{n}}$ is maximized when $N_h^K(s,a) = \frac{N}{SA}$ for every state-action pair.
Finally, we have
\begin{align*}
\textup{Regret}(K) \leq O\left(H^{2} S A+cH\sqrt{T\iota} +\sqrt{dH^4SAT\iota}\right).
\end{align*}
The total regret of $Q$-learning with weighted bonus is at most $O(\sqrt{dH^4SAT\iota})$. We complete the proof.
\end{proof}

\section{Details of Algorithm Implementation and Experimental Settings}

\subsection{Details of R3L Environments}

In general, we divide real-world decision-making tasks with non-replenishable resources into two categories. In the first category of tasks, all actions consume resources and different actions consume different quantities of resources, such as robotic control with limited energy. In these tasks, the agent needs to seek actions that achieve high rewards while consuming small quantities of resources. In the second category, only specific actions consume resources, such as video games with consumable items. In these tasks, the agent needs to seek proper states to consume the resources. 

To evaluate popular RL methods in both kinds of R3L tasks, we design three series of environments with limited resources based on Gym \cite{gym} and Mujoco \cite{mujoco}. The first is the autonomous electric robot task. In this task, the resource is electricity and all actions consume electricity. The quantity of consumed electricity depends on the amplitude of the action $a$, defined by $0.1*\|a\|_2^2$. 
The second is the robotic delivery task. In this task, the resource is goods and only the ``unload'' action consumes the goods. The agent needs to ``unload'' the goods at an unknown destination. The third is a task that combines the first and the second. In this task, resources are goods and electricity. The agent needs to ``unload'' the goods at an unknown destination with limited electricity. We provide detailed descriptions of these environments in the following. 

\noindent \textbf{Autonomous electric robot tasks}

\textbf{States} We augment the states with the quantity of the remaining electricity. For example, the state space of Electric Mountain Car is the same as the state space of Mountain Car, except for adding one dimension to represent the quantity of the remaining electricity. In terms of the initial quantity of electricity, we use Surprise in tasks without resources and record the total quantity of electricity $r_0$ consumed in an episode when Surprise converges. Intuitively, $r_0$ represents the total quantity of electricity needed for an expert to complete the task. For Electric tasks, we set the initial quantity of electricity $2r_0$. Specifically, the quantities of electricity are 12, 140, and 32 for Electric Mountain Car, Electric Ant, and Electric Half-Cheetah, respectively.

\textbf{Actions} The action space remains unchanged and the consumed electricity is $0.1 * \|a\|_2^2$ after the agent takes action $a$. That is, the consumed electricity positively correlates with the amplitude of actions. 

\textbf{Rewards and Terminal States} 
The rewards are 0 until the agent reaches the destination. The reward is $100 + 100 * \frac{u}{I_{\max}}$, where $u$ is the remaining electricity when reaching the destination and $I_{\max}$ is the quantity of initial electricity. When the agent exhausts electricity or reaches the destination, the environment will be reset. For Electric Ant and Electric Half-Cheetah, 
the position of the agent is denoted by the coordinates $(x,y,h)$. We define the destination by the region $x = 4 $ and $x = 9$, respectively. For Electric Mountain Car, the destination is the top of the hill.

\noindent \textbf{Robotic delivery tasks}

\textbf{States} We augment the states with the quantity of the remaining goods. For example, the state space of Delivery Mountain Car is the same as the state space of Mountain Car, except for adding one dimension to represent  the number of remaining goods. For Delivery Mountain Car, the number of goods is ten at the beginning of each episode. For Delivery Ant and Delivery Half-Cheetah, the number of goods is four at the beginning of each episode.

\textbf{Actions} We augment the actions with the number of goods to be unloaded. For example, the action space of Delivery Mountain Car is the same as the action space of Mountain Car, except for adding one dimension to represent the number of goods to be unloaded. For Delivery Mountain Car, the agent can unload $u$ units of goods at each step, $u\in[0,1]$. For Delivery Ant and Delivery Half-Cheetah, the agent unloads one or zero unit of goods at each step.



\textbf{Rewards and Terminal States} 
The rewards are 0 until the agent ``unloads" the goods at the destination. For Delivery Mountain Car, the destination is the top of the hill. The reward is $100\cdot u$, where $u$ is the number of goods unloaded at the top of the hill. When the agent reaches the top of the hill and carries no goods, the environment will be reset. For Delivery Ant and Delivery Half-Cheetah, the position of the agent is denoted by the coordinates $(x,y,h)$. We define the destination by the region $x \in [4,5]$. The reward is 100 if the agent reaches the destination and ``unloads" the goods. When the agent enters the destination and ``unloads" the goods, the environment will be reset.

\noindent \textbf{Robotic delivery tasks with limited electricity}

\textbf{States} We augment the states with the quantities of the remaining electricity and the remaining goods. For example, the state space of Electric Delivery Mountain Car is the same as the state space of Mountain Car, except for adding two dimensions to represent the quantities of the remaining electricity and the remaining goods. We set the initial quantity of electricity the same as that in the corresponding autonomous electric robot tasks. We set the initial quantity of goods the same as that in the corresponding robotic delivery tasks. Specifically, the quantities of electricity are 12, 140, and 32 for Electric Delivery Mountain Car,   Electric Delivery Ant, and Electric Delivery Half-Cheetah, respectively. The quantities of goods are 10, 4, and 4 for Electric Delivery Mountain Car,   Electric Delivery Ant, and Electric Delivery Half-Cheetah, respectively.

\textbf{Actions} 
First, the consumed electricity is $0.1 * \|a\|_2^2$ after the agent takes action $a$. That is, the consumed electricity positively correlates with the amplitude of actions. Second, we augment the actions with the quantity of goods to be unloaded. For example, the action space of Electric Delivery Mountain Car is the same as the action space of Mountain Car, except for adding one dimension to represent the number of goods to be unloaded. For Electric Delivery Mountain Car, the agent can unload $u$ units of goods at each step, $u\in[0,1]$. For Electric Delivery Ant and Electric Delivery Half-Cheetah, the agent unloads one or zero unit of goods at each step.

\textbf{Rewards and Terminal States}
The rewards are 0 until the agent reaches the destination and ``unloads" the goods at the destination. The reward is $100 + 100 * \frac{u}{I_{\max}}$, where $u$ is the remaining electricity when delivering the goods to the destination and $I_{\max}$ is the quantity of initial electricity. When the agent exhausts electricity or delivery the goods to the destination, the environment will be reset. For Electric Delivery Ant and Electric Delivery Half-Cheetah, 
the position of the agent is denoted by the coordinates $(x,y,h)$. We define the destination by the region $x = 4 $ and $x = 9$, respectively. For Electric Delivery Mountain 
Car, the destination is the top of the hill.

\begin{table}[t]
	\caption{Shared parameters used in all experiments.}\label{hyper:shared}
	\centering 
	\vspace{0.3cm}
	\begin{tabular}{ l  l }
		\toprule
		Parameter  & Value  \\ 
		\midrule 
		optimizer &Adam \\
		learning rate  & $ 3\cdot10^{-4} $\\
		discount ($ \gamma $)&0.99\\
		replay buffer size&$ 10^6 $\\
		number of hidden layers &2 \\
		number of samples per minibatch &256\\
		nonlinearity & ReLU\\
		target smoothing coefficient &0.005\\
		target update interval & 1\\
		\bottomrule
	\end{tabular}
\end{table}

\begin{table}[t]
	\caption{Hyperparameters for models in Electric/Delivery Ant and Electric/Delivery Half-Cheetah environments.}\label{hyper:models_high_dimensional}
	\centering 
	\vspace{0.3cm}
	\begin{tabular}{ l  l }
		\toprule
		Parameter  & Value  \\ 
		\midrule 
		ensemble size & 1 (Surprise/ RAEB) \\ 
		&  32 (JDRX) \\
		hidden layers  & 4\\
		hidder layer size & 512 \\
		batch size & 256 \\
		non-linearity & Swish \\
		learning rate & $3 \times 10^{-4}$\\
		\bottomrule
	\end{tabular}
\end{table}

\begin{table}[t]
	\caption{Hyperparameters for models in Electric/Delivery Mountain Car environment.}\label{hyper:models_low_dimensional}
	\centering 
	\vspace{0.3cm}
	\begin{tabular}{ l  l }
		\toprule
		Parameter  & Value  \\ 
		\midrule 
		ensemble size & 1 (Surprise/ RAEB) \\ 
		&  32 (JDRX) \\
		hidden layers  & 1\\
		hidder layer size & 32 \\
		batch size & 256 \\
		non-linearity & Swish \\
		learning rate & $3 \times 10^{-4}$\\
		\bottomrule
	\end{tabular}
\end{table}

\begin{figure}[t]
    \centering
    \includegraphics[width=220pt]{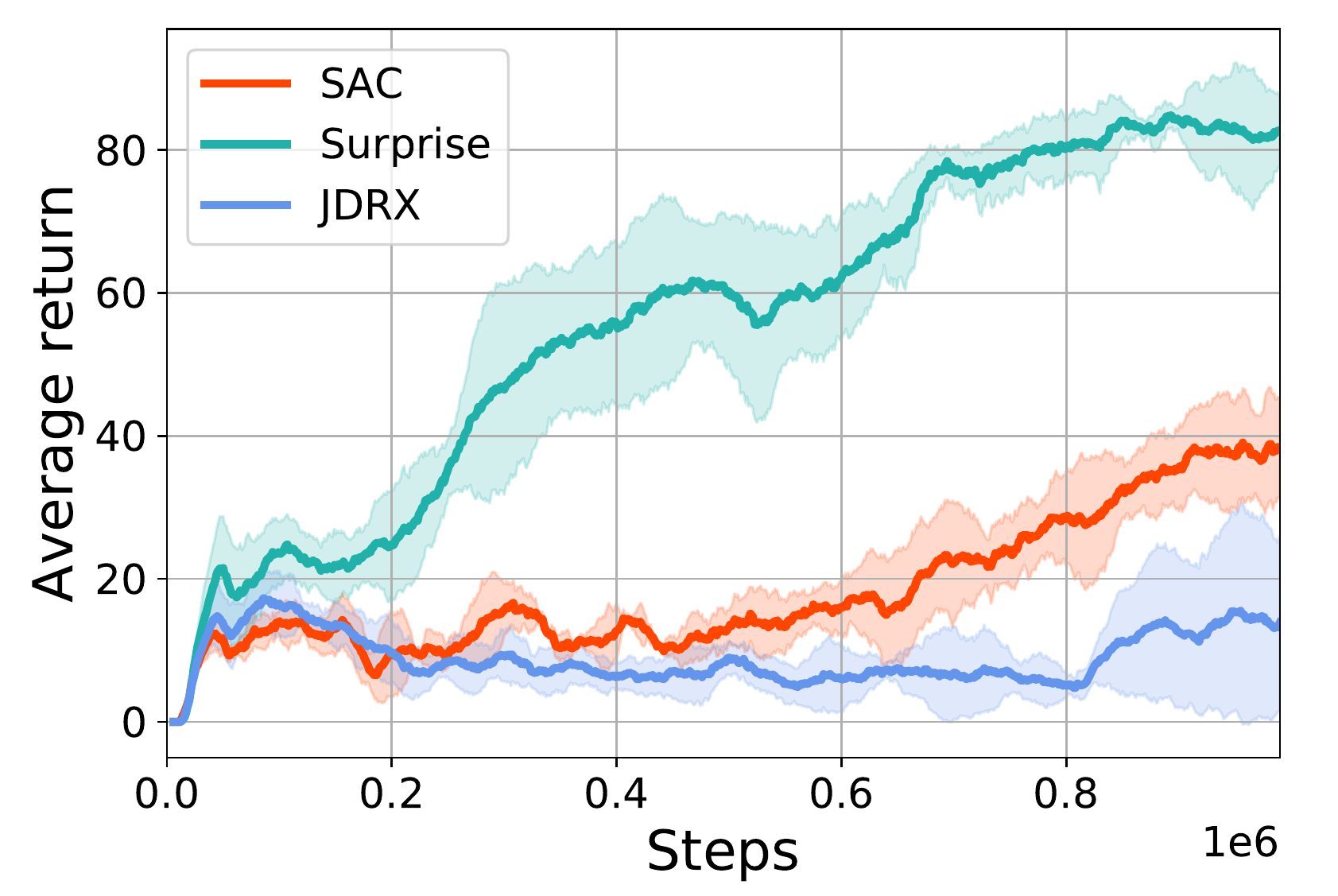}
    \caption{We compare Surprise with JDRX and SAC in Ant Corridor environment. The solid curves correspond to the mean and the shaded region to the standard deviation over 3 random seeds. Moreover, We smooth curves uniformly for visual clarity.}
    \label{fig:reproducing}
\end{figure}

\begin{figure}[t]
    \centering
    \begin{subfigure}{0.48\columnwidth}
        \includegraphics[width=\textwidth]{./figures/analysis/difficulty_mountaincar_200000.pdf}
        \subcaption{Exploration behavior of baselines on Delivery Mountain Car.}
        \label{fig:max_height_baselines}
    \end{subfigure}
    \begin{subfigure}{0.48\columnwidth}
        \includegraphics[width=\textwidth]{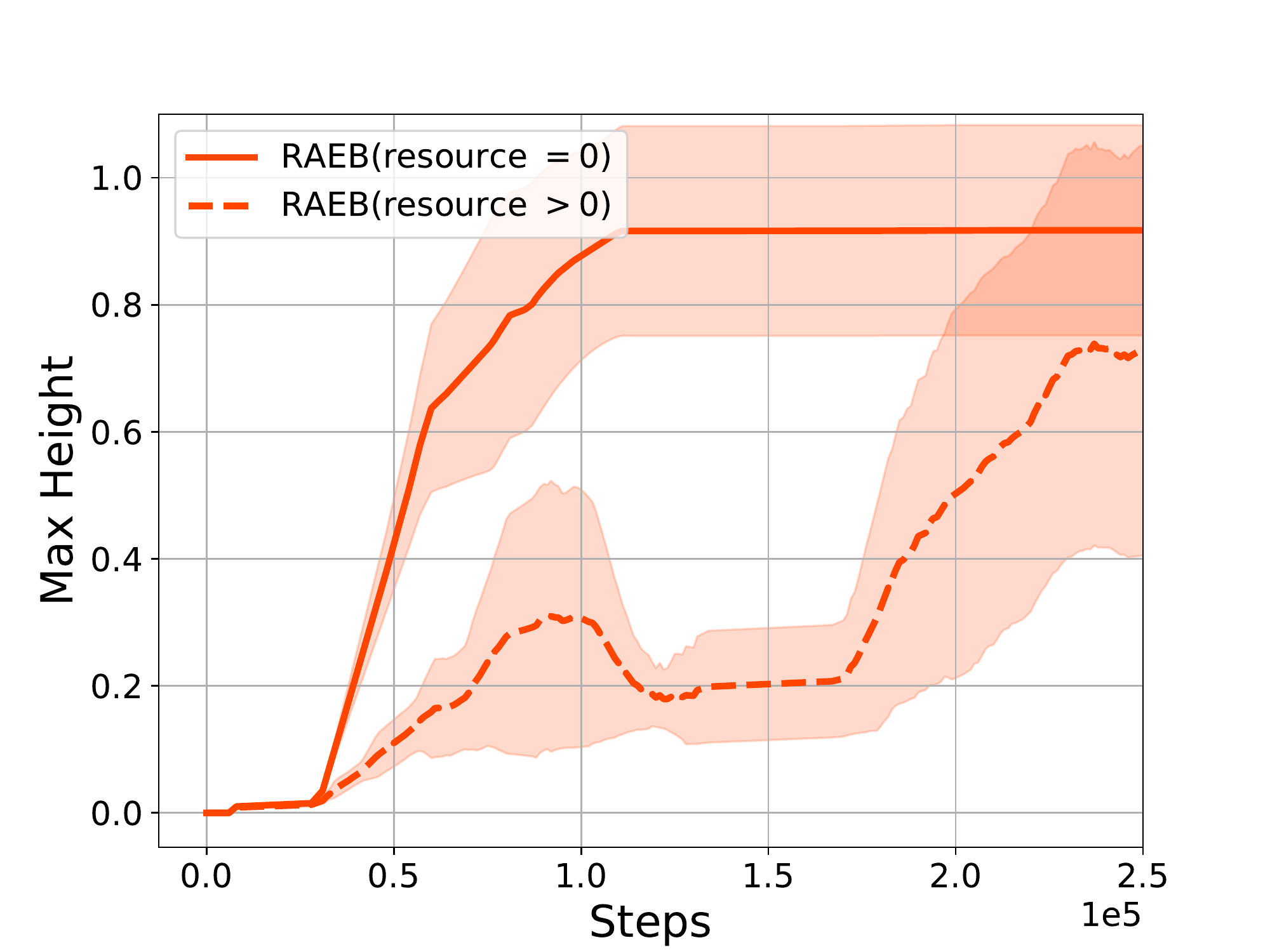}
        \subcaption{Exploration behavior of RAEB on Delivery Mountain Car.}
        \label{fig:max_height_raeb}
    \end{subfigure}
    \caption{The max height during training in the Delivery Mountain Car environment of RAEB and baselines. The solid curves correspond to the max height that the car reaches in an episode no matter whether it exhausts resources or not, while the dashed curves correspond to the max height the car reaches before exhausting the resources.}
    \label{fig:max_height}
\end{figure}

\begin{table}[t]
    \centering
    \caption{The improvement of the sample efficiency of RAEB compared with the baselines on Delivery Ant. \#Steps denotes the number of steps required to achieve the corresponding average return. RAEB achieves the average return 11.13 in the first $1\times 10^5$ steps, while the best baselines achieves the average return 5.65 in the first $1\times 10^6$ steps.}
    \label{tab:improvement_of_sample_efficiency}
    \vspace{3mm}
    \begin{tabular}{ccc}
        \toprule
         \textbf{Algorithm} & \textbf{The average return} & \textbf{\#Steps} \\
         \midrule
         RAEB & 11.13 &  $\mathbf{1 \times 10^5}$\\
         SAC & 2.15 & $\mathbf{1 \times 10^6}$\\
         Surprise & 1.25 &  $\mathbf{1 \times 10^6}$\\
         JDRX & 0.1 & $\mathbf{1 \times 10^6}$\\
         SimHash & 0.2 & $\mathbf{1 \times 10^6}$\\
         NovelD & 0.1 & $\mathbf{1 \times 10^6}$\\
         Lagrangian & 5.65 & $\mathbf{1 \times 10^6}$\\
         \bottomrule
    \end{tabular}
\end{table}

\begin{table}[t]
    \centering
    \vspace{-2mm}
    \caption{The average time steps when exhausting the resources for PPO, SAC, and Surprise in each episode on Delivery Mountain Car. This table shows that PPO, SAC, and Surprise exhaust resources in much fewer steps than our proposed approach (RAEB).}
    \label{tab:avg_time_steps_exhausting_resources}
    \vspace{3mm}
    \begin{tabular}{cc}
        \toprule
         \textbf{Algorithm} & \textbf{The average time steps} \\
         \midrule
         PPO & 19.74 \\
         SAC & 72.7 \\
         Surprise & 25.8 \\
         \textbf{RAEB} & \textbf{118.33} \\
         \bottomrule
    \end{tabular}
\end{table}

\begin{figure*}[ht] 
    \centering 
    \includegraphics[width=0.90\textwidth]{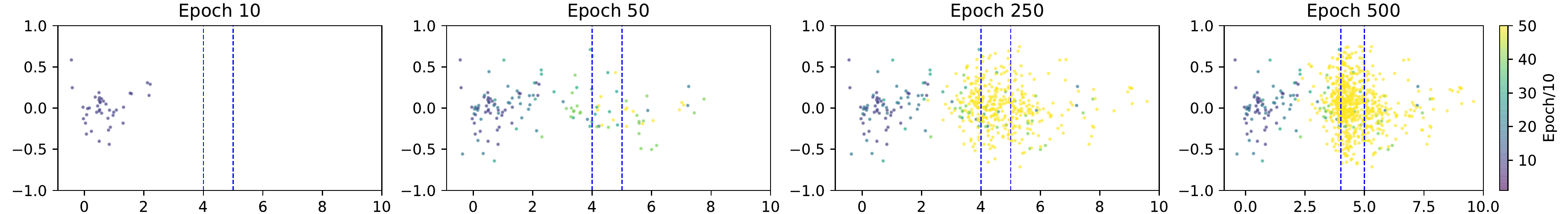}
    \caption{Illustration of RAEB exploration in the Delivery Ant environment. Colored points illustrate states where the agent consumes resources achieved by the policy after every one epoch. The color indicates the number of epochs. The region between the two dashed lines is the destination.}
    \label{fig:resource_ant_corridor_visualization}
\end{figure*}

\begin{figure*}[ht]
    \centering
    \begin{subfigure}{0.23\textwidth}
        \includegraphics[width=\textwidth]{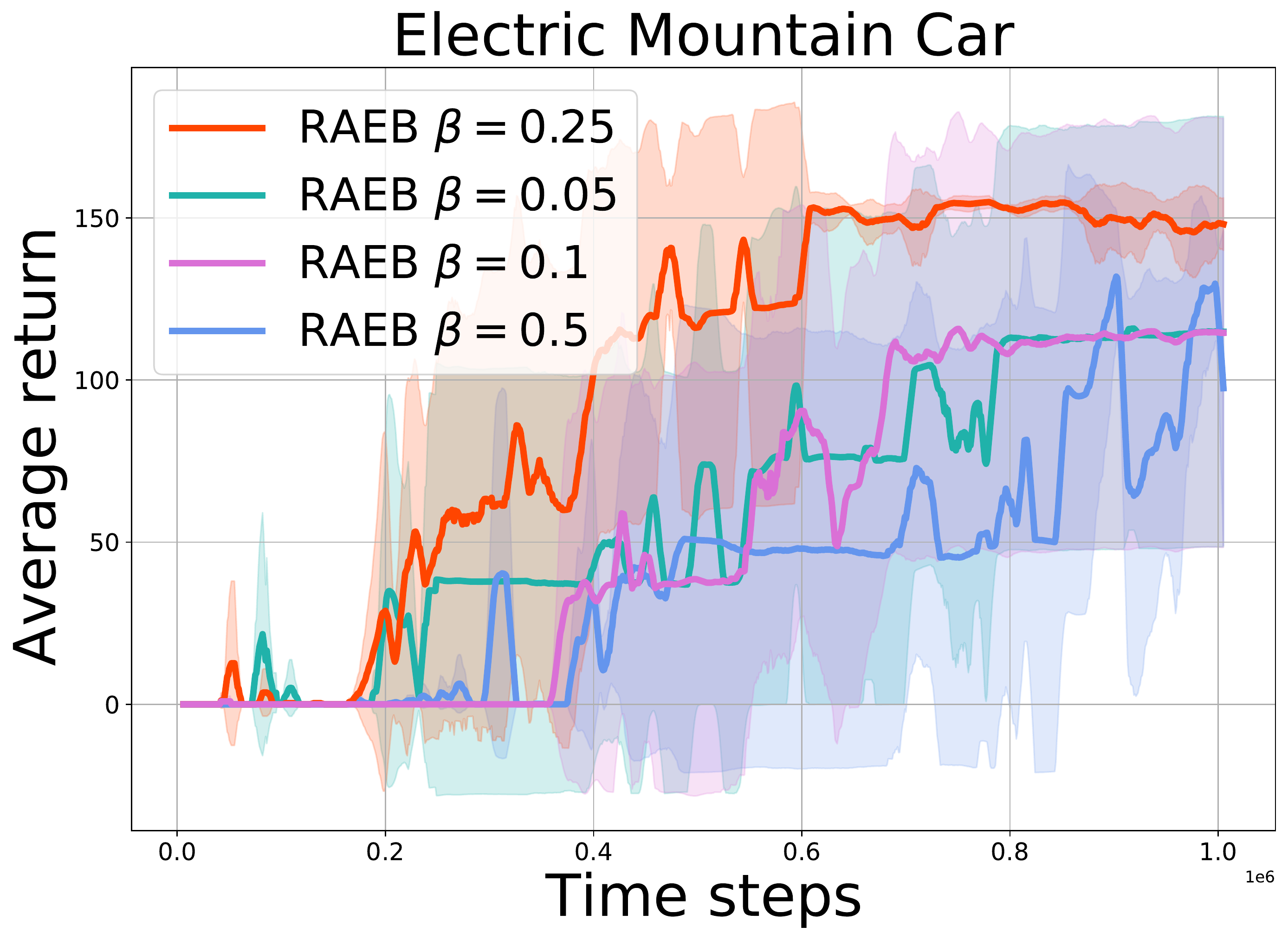}
        \subcaption{Sensitivity analysis of RAEB to the hyperparameter $\beta$.}
        \label{fig:sensitivity_beta_electric_car}
    \end{subfigure}
    \begin{subfigure}{0.23\textwidth}
        \includegraphics[width=\textwidth]{./aaai23_figures/sensitivity/beta_delivery_ant.pdf}
        \subcaption{Sensitivity analysis of RAEB to the hyperparameter $\beta$.}
        \label{fig:sensitivity_alpha_electric_car}
    \end{subfigure}
    \begin{subfigure}{0.23\textwidth}
        \includegraphics[width=\textwidth]{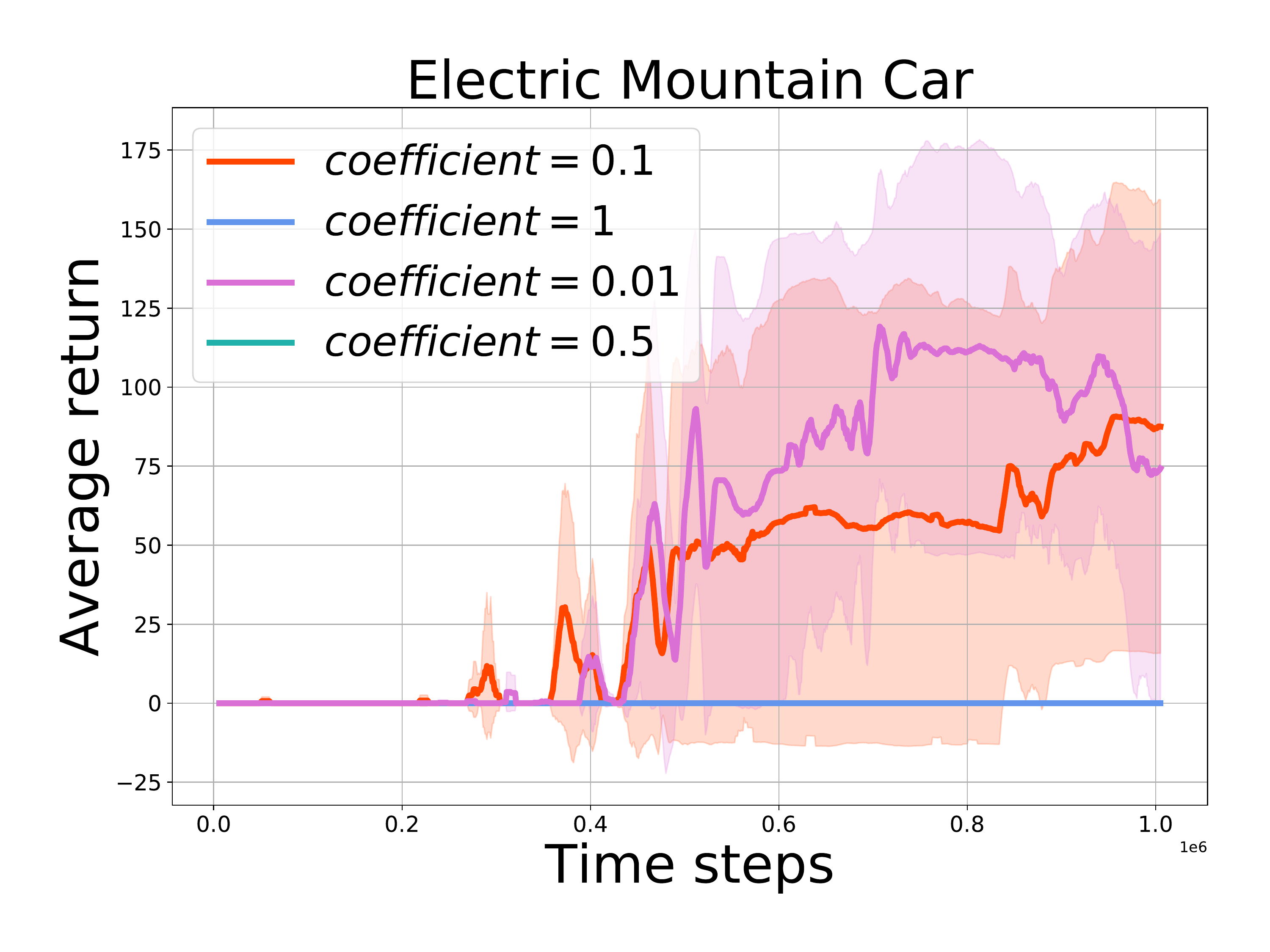}
        \subcaption{Sensitivity analysis of SurpriseRB to $\beta$.}
        \label{fig:sensitivity_surpriserb_electric_car}
    \end{subfigure}
    \begin{subfigure}{0.23\textwidth}
        \includegraphics[width=\textwidth]{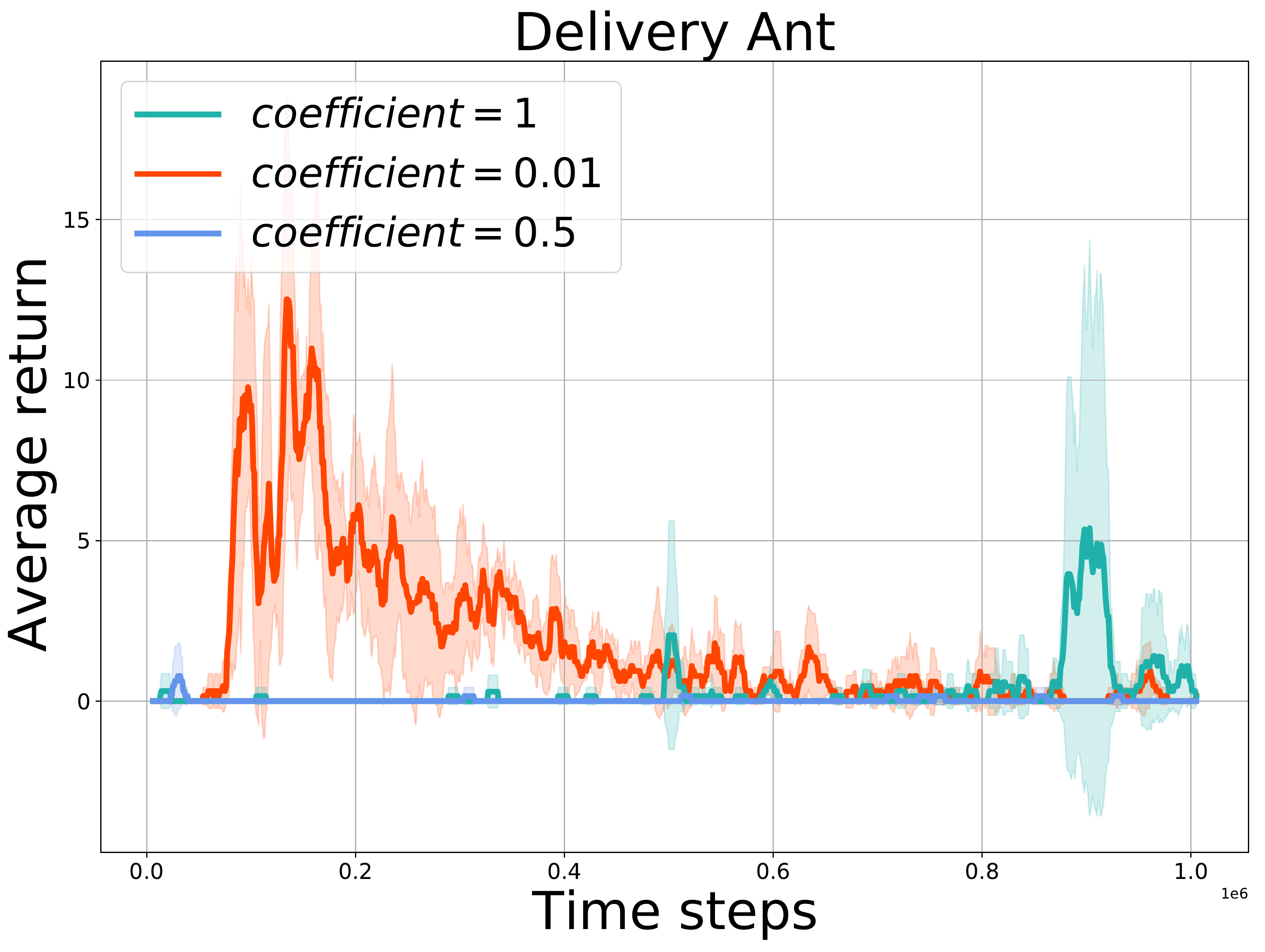}
        \subcaption{Sensitivity analysis of SurpriseRB to  $\beta$.}
        \label{fig:sensitivity_surpriserb_delivery_ant}
    \end{subfigure}
    \vspace{-3mm}
    \caption{Performance of RAEB and SurpriseRB with different $\beta$ on Electric Mountain Car and Delivery Ant.}
    \label{fig:sensitivity_analysis}
\end{figure*}

\begin{figure*}[!ht]
    \centering
    \begin{subfigure}{0.24\textwidth}
        \includegraphics[width=\textwidth]{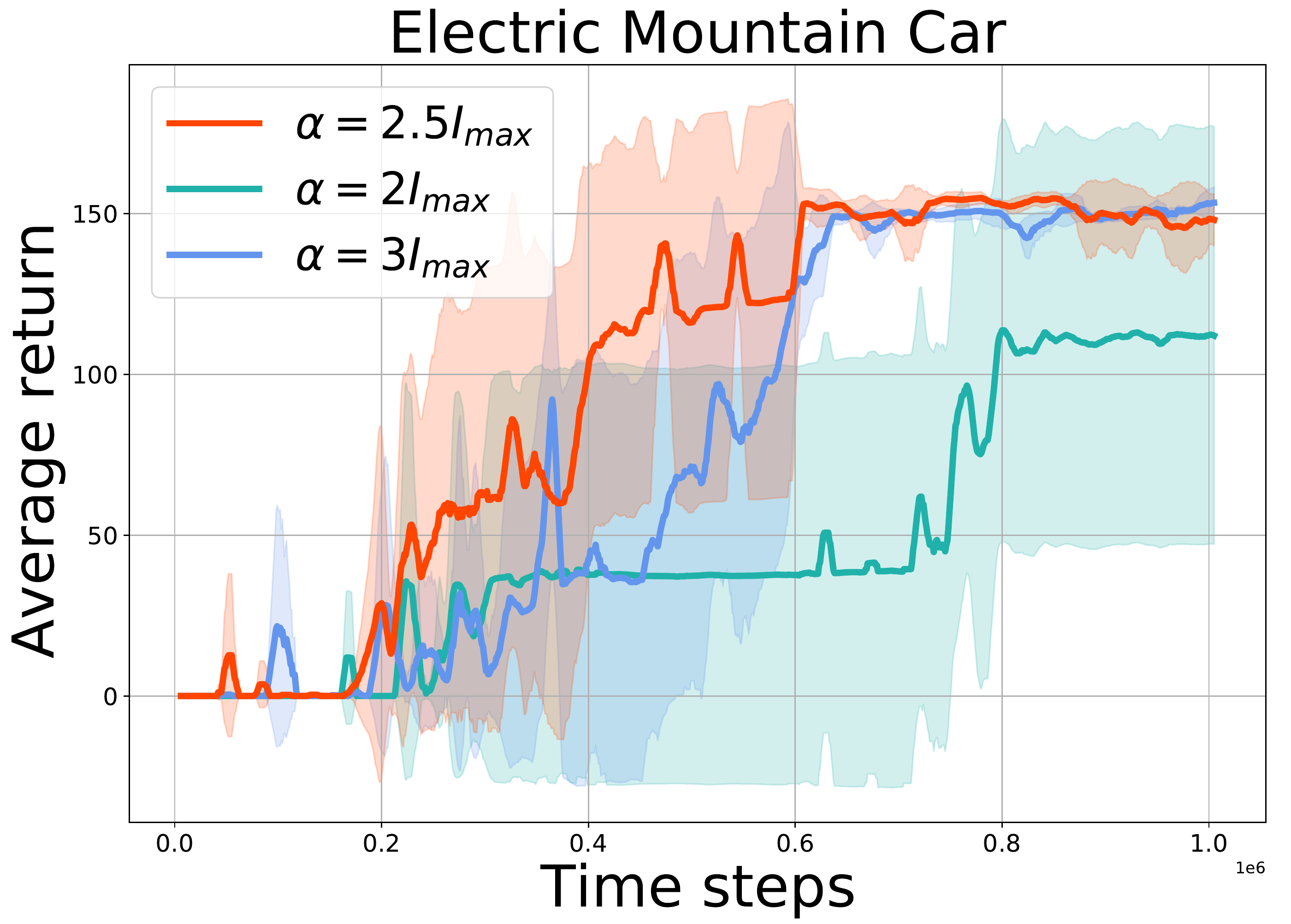}
        \subcaption{Sensitivity analysis (RAEB $\alpha$)}
        \label{fig:sensitivity_alpha_electric_car}
    \end{subfigure}
    \begin{subfigure}{0.24\textwidth}
        \includegraphics[width=\textwidth]{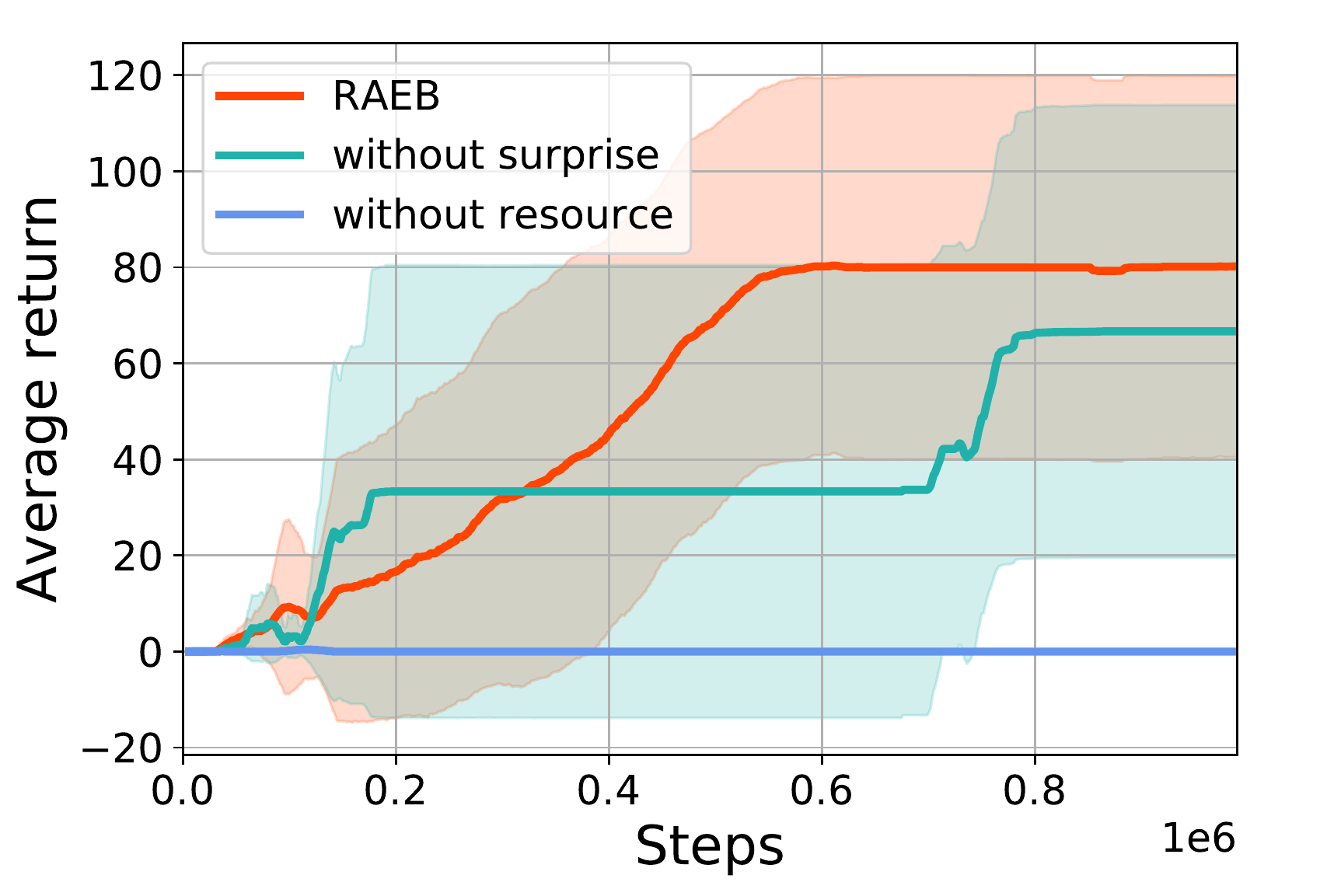}
        \subcaption{Delivery Half-Cheetah}
        \label{fig:component_cheetah}
    \end{subfigure}
    \begin{subfigure}{0.24\textwidth}
        \includegraphics[width=\textwidth]{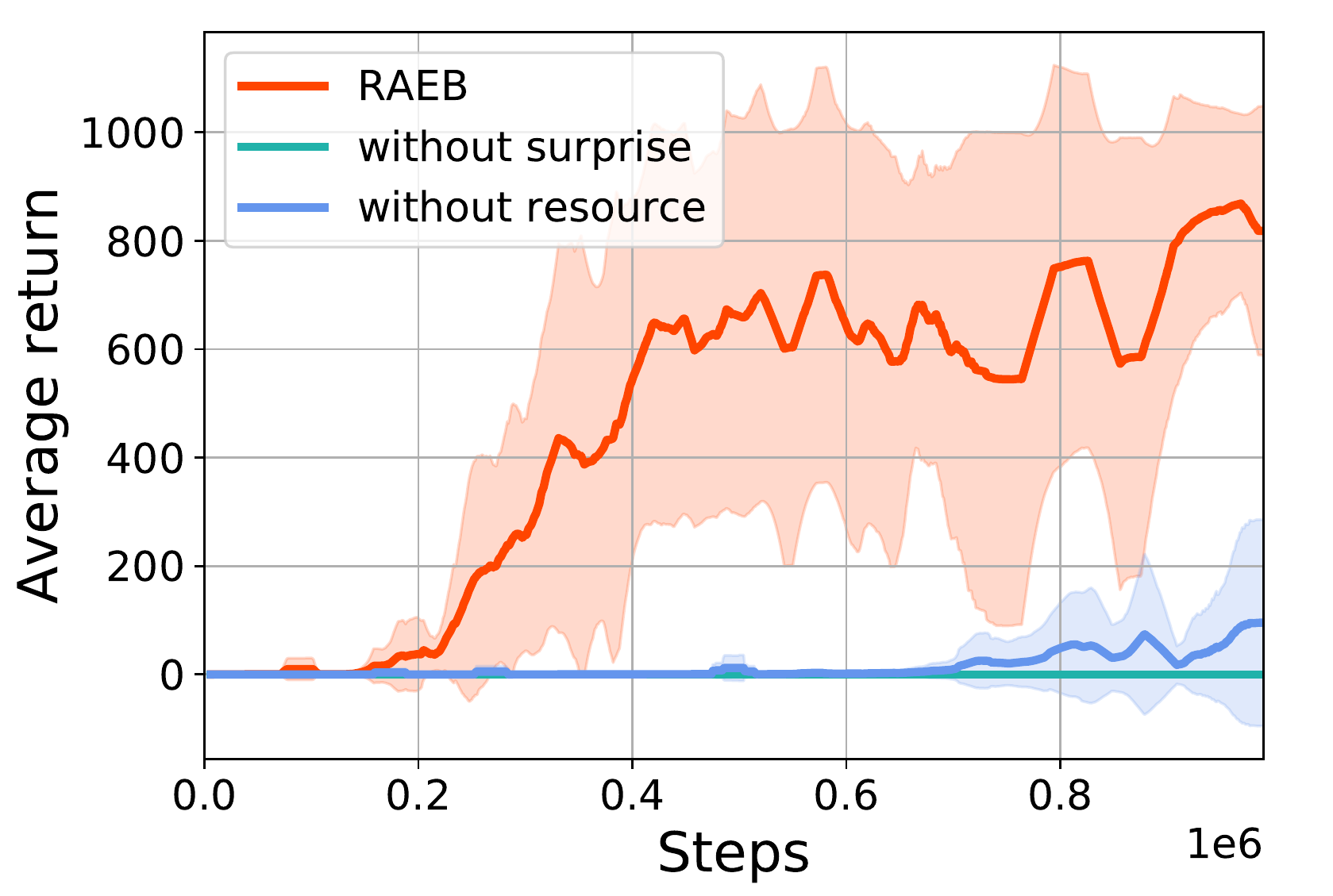}
        \subcaption{Delivery Mountain Car}
        \label{fig:component_mountaincar}
    \end{subfigure}
    \begin{subfigure}{0.24\textwidth}
        \includegraphics[width=\textwidth]{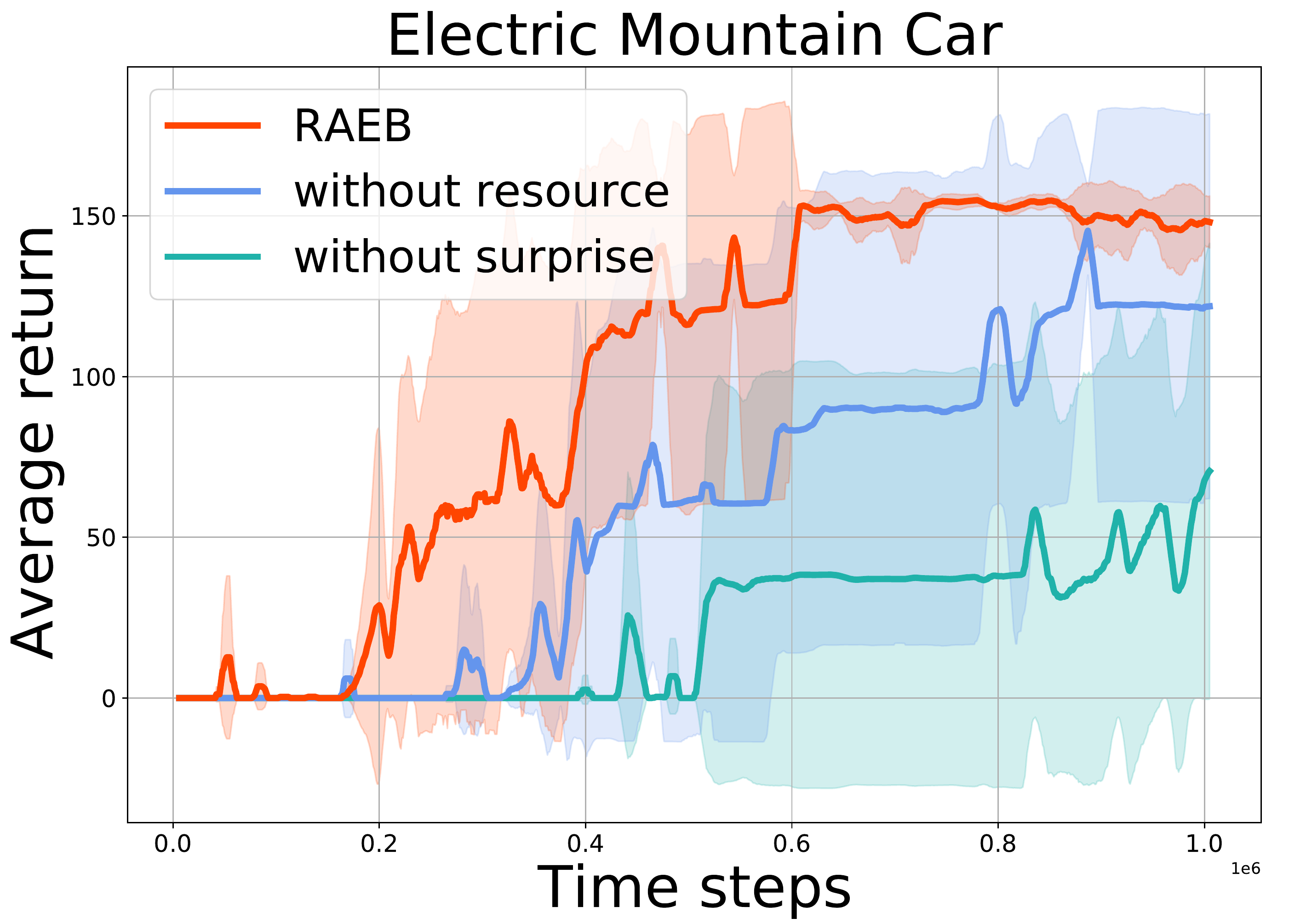}
        \subcaption{Electric Mountain Car}
        \label{fig:component_electric_car}
    \end{subfigure}
    \vspace{-3mm}
    \caption{(a) Sensitivity analysis of RAEB to $\alpha$ on Electric Mountain Car. (b)-(d) Performance of RAEB, RAEB without Surprise, and RAEB without resources on Delivery Half-Cheetah, Delivery Mountain Car, and Electric Mountain Car.}
    \label{fig:component_cheetah_car}
\end{figure*}

\subsection{Details of Experimental Settings and Hyperparameters}

\textbf{Implementation of our Baselines} We list the implementations of our baselines, including Surprise, Jensen-R\'enyi Divergence Reactive Exploration (JDRX), proximal policy optimization (PPO), and soft actor critic (SAC). Except for proximal policy optimization (PPO), we implemented these algorithms in a unified code framework for a fair comparison. Moreover, we use the PyTorch implementation of PPO in \url{https://github.com/openai/spinningup}. 

\textbf{Hyperparameters} 
We list common parameters used in comparative evaluation and ablation study in Table \ref{hyper:shared}. We use the same hyperparameters as SAC \cite{pmlr-v80-haarnoja18b} if possible in all experiments. For forward dynamics models, we use the same hyperparameters as those of JDRX \cite{max} if possible. For Electric/Delivery Ant and Electric/Delivery Half-Cheetah tasks,
we use a neural network policy with two layers of 128 units and a q-value network with two layers of 256 units. We list the Hyperparameters for Models in Electric/Delivery Ant and Electric/Delivery Half-Cheetah environments in Table \ref{hyper:models_high_dimensional}. For Electric/Delivery Mountain Car task, we use a neural network policy with one layer of 32 units and a q-value network with one layer of 32 units. We list the Hyperparameters for Models in Electric/Delivery Mountain Car in Table \ref{hyper:models_low_dimensional}.

\subsection{Hardware} We train and evaluate all methods on a single machine that contains eight GPU devices (NVidia GeForce GTX 2080 Ti) and 16 Intel Xeon CPU E5-2667 v4 CPUs.

\section{Additional Experimental Results}

\subsection{Choosing our base algorithm}

To evaluate the exploration ability of existing exploration strategies, we design a task where the agent (an ant) starts at the left corner of a corridor and aims to reach an unknown destination. The agent can get the nonzero reward $100$, when the x-coordinate $ > 4$. Otherwise, the agent receives zero rewards. Once the agent gets nonzero rewards, the environment will be reset. We call the environment Ant Corridor. Figure \ref{fig:reproducing} shows that Surprise significantly outperforms JDRX and SAC in the Ant Corridor environment. Therefore, we choose Surprise as the base algorithm of our method. 

\subsection{The sample efficiency of RAEB compared with baselines} Table \ref{tab:improvement_of_sample_efficiency} shows the improvement of the sample efficiency of RAEB compared with the baselines on Delivery Ant. The results show that RAEB improves the sample efficiency by up to an order of magnitude compared to baselines. 

\subsection{More results about the exploration behavior of popular RL methods} Table \ref{tab:avg_time_steps_exhausting_resources} shows the average time steps when exhausting the resources for PPO, SAC, and Surprise in each episode on Delivery Mountain Car. The results show that PPO, SAC, and Surprise exhaust resources in much fewer steps than our proposed approach (RAEB).

\subsection{More Results for Illustrating the Exploration of RAEB}

\textbf{Analyzing the learning behavior of RAEB} 
Figure \ref{fig:max_height} shows that RAEB significantly reduces unnecessary resource-consuming trials compared to baselines, and thus explores resource-consuming actions efficiently.


\noindent\textbf{Delivery Ant environment}
Figure \ref{fig:resource_ant_corridor_visualization} visualizes the position where the ant consumes resources in the training process. The results demonstrate that RAEB can efficiently explore the environments, especially exploring the resource-consuming actions.

\subsection{More Results for Ablation Study}

\subsubsection{Sensitivity analysis to $\beta$ of RAEB and SurpriseRB}

We provide additional results (Figure \ref{fig:sensitivity_analysis}) as a supplement to Section 6.2 in the main text. The results show that RAEB is insensitive to the hyperparameters $\beta$. Moreover, the results in Figure \ref{fig:sensitivity_analysis} show that SurpriseRB is very sensitive to $\beta$, which demonstrates the superiority of RAEB over SurpriseRB. 

\subsubsection{Sensitivity analysis to $\alpha$ of RAEB}
We provide additional results (Figure \ref{fig:sensitivity_alpha_electric_car}) as a supplement to Section 6.2 in the main text. The results in Figure \ref{fig:sensitivity_alpha_electric_car} show that RAEB is insensitive to $\alpha$ on Electric Mountain Car.

\subsubsection{Contribution of each component}
We provide additional results (Figures \ref{fig:component_cheetah}, \ref{fig:component_mountaincar}, and \ref{fig:component_electric_car}) as a supplement to Section 6.2 in the main text. The results show that each individual component of RAEB is significant for efficient exploration on Delivery Half-Cheetah, Delivery Mountain Car, and Electric Ant.

\section{Additional Details}

\subsection{Projection trick}
To ensure any policy satisfies resource constraints, we project the consumed resources by resource-consuming actions into the interval $[0,I(s)]$, where $s$ is the current state and $I(s)$ represents the remaining resources of the agent. For example, suppose $I(s) = 0.5$ and $a_u = 1$, then the projected resource-consuming action will be $0.5$.

\end{document}